\documentclass[final]{article}

\usepackage[main]{neurips_2025}

\usepackage[utf8]{inputenc} 
\usepackage[T1]{fontenc}    
\usepackage{hyperref}       
\usepackage{url}            
\usepackage{booktabs}       
\usepackage{amsfonts}       
\usepackage{nicefrac}       
\usepackage{microtype}      
\usepackage{xcolor}         

\usepackage{graphicx}
\usepackage{natbib}
\usepackage{amsmath}
\usepackage{algorithm}
\usepackage{algorithmic}
\usepackage[table]{xcolor}

\usepackage{amsthm}

\newtheorem{theorem}{Theorem}
\newtheorem{lemma}{Lemma}

\theoremstyle{definition}

\theoremstyle{remark}
\newtheorem{remark}{Remark}

\newcommand{\AR}{\operatorname{AR}}
\title{\texorpdfstring{From $O(mn)$ to $O(r^2)$: Two-Sided Low-Rank Communication
for Adam in Distributed Training with Memory Efficiency}{From O(mn) to O(r^2): Two-Sided Low-Rank Communication for Adam}}

\author{
  Sizhe Dang\thanks{equal contribution.} \\
  Xi'an Jiaotong University\\
  \texttt{darknight1118@stu.xjtu.edu.cn} 
  \And
  Jiaqi Shao\footnotemark[1] \\
  Xi'an Jiaotong University \\
  \texttt{shaojiaqi@stu.xjtu.edu.cn} 
  \And
  Xiaodong Zheng \\
  Xi'an Jiaotong University \\
  \texttt{zxd\_xjtu@stu.xjtu.edu.cn} 
  \And
  Guang Dai \\
  SGIT AI Lab \\
  \texttt{guang.dai@gmail.com} 
  \And
  Yan Song \\
  University of Science and Technology of China \\
  \texttt{clksong@gmail.com} 
  \And
  Haishan Ye\thanks{Corresponding author.} \\
  Xi'an Jiaotong University \\
  \texttt{yehaishan@xjtu.edu.cn} 
}

\begin{document}

\makeatletter
\renewcommand{\@noticestring}{}
\makeatother
\maketitle

\begin{abstract}
As foundation models continue to scale, pretraining increasingly relies on data-parallel distributed optimization, making bandwidth-limited gradient synchronization a key bottleneck.
Orthogonally, projection-based low-rank optimizers were mainly designed for memory efficiency, but remain suboptimal for communication-limited training: one-sided synchronization still transmits an $O(rn)$ object for an $m\times n$ matrix gradient and refresh steps can dominate peak communicated bytes.
We propose \textbf{TSR}, which brings \textbf{t}wo-\textbf{s}ided low-\textbf{r}ank communication to Adam-family updates (\textbf{TSR-Adam}) by synchronizing a compact core $U^\top G V\in\mathbb{R}^{r\times r}$, reducing the dominant per-step payload from $O(mn)$ to $O(r^2)$ while keeping moment states in low-dimensional cores.
To further reduce the peak communication from subspace refresh, TSR-Adam adopts a randomized SVD-based refresh that avoids full-gradient synchronization.
We additionally extend low-rank communication to embedding gradients with embedding-specific ranks and refresh schedules, yielding additional communication and memory savings over keeping embeddings dense.
Across pretraining from 60M to 1B model scales, TSR-Adam reduces average communicated bytes per step by $13\times$, and on GLUE fine-tuning it reduces communication by $25\times$, while achieving comparable performance; we further provide a theoretical stationarity analysis for the proposed update. Code is available at \url{https://github.com/DKmiyan/TSR-Adam}.
\end{abstract}

\begin{figure}[th]
  \centering
  \begin{minipage}{0.32\columnwidth}
    \centering
    \includegraphics[width=\linewidth]{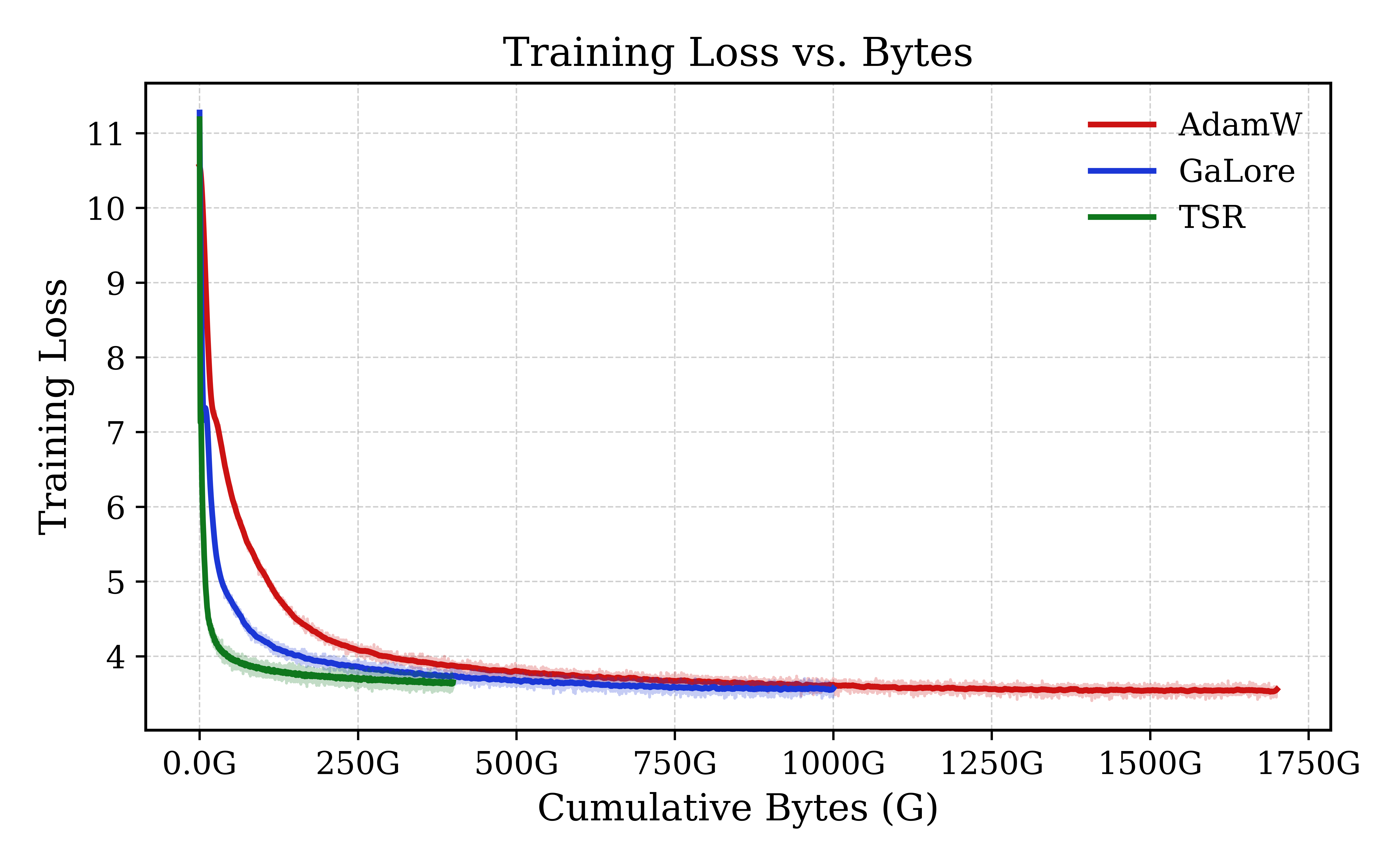} 
    (a)LLaMA-60M
  \end{minipage}
  \hfill
  \begin{minipage}{0.32\columnwidth}
    \centering
    \includegraphics[width=\linewidth]{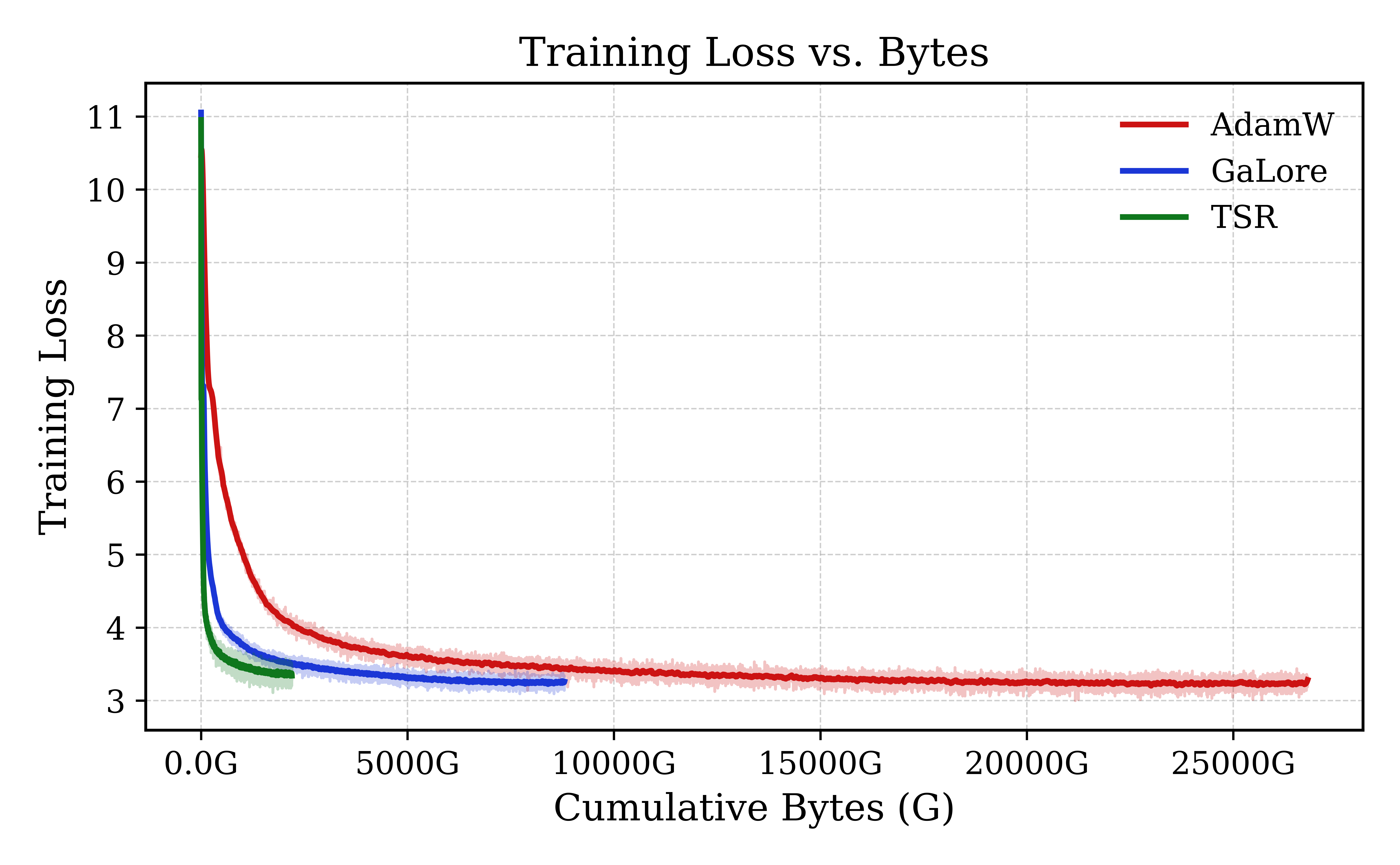} 
    (b)LLaMA-130M
  \end{minipage}
  \hfill
  \begin{minipage}{0.32\columnwidth}
    \centering
    \includegraphics[width=\linewidth]{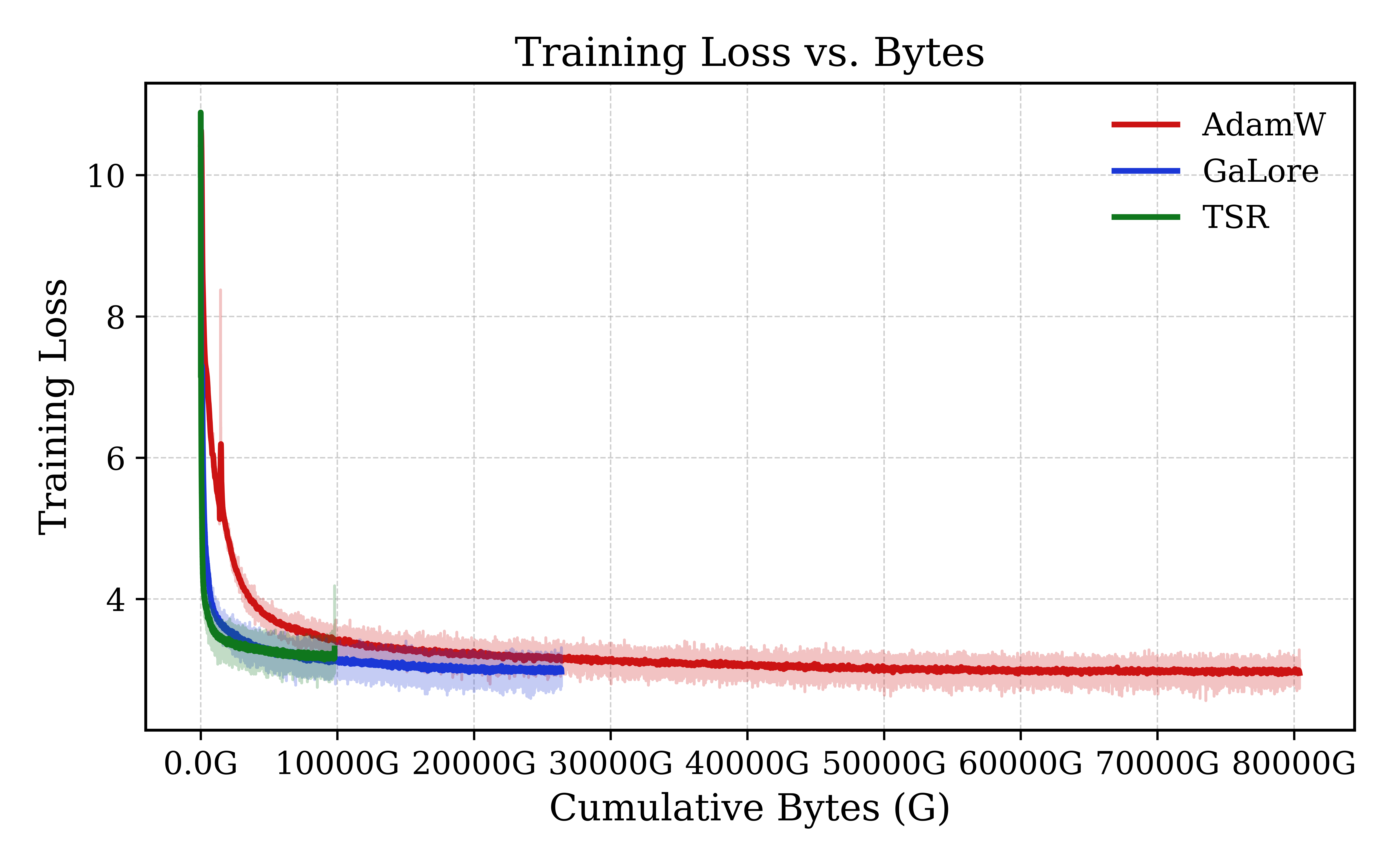} 
    (c)LLaMA-350M
  \end{minipage}
  \caption{
    \textbf{Bytes-to-Loss.} Training loss as a function of cumulative communicated bytes for representative model scales.
    TSR-Adam reaches lower loss under the same communication budget compared with baselines.
    (a)--(c) correspond to three representative model scales.
  }
  \label{fig:bytes_to_loss}
\end{figure}

\section{Introduction}

\label{sec:intro}

Large-scale pretraining has established itself as the cornerstone of modern foundation models~\cite{brown2020language,bommasani2021opportunities,touvron2023llama}.
However, practical training at these scales relies heavily on distributed data-parallel execution, where synchronizing gradients across many accelerators can become a substantial bottleneck.
This limitation is exacerbated by the disparity between high-bandwidth on-node interconnects and lower-bandwidth cross-node links (e.g., NVLink vs.~PCIe), so communication overhead can quickly dominate once synchronization traverses bandwidth-limited boundaries.
As model sizes and cluster scales grow, reducing the \emph{gradient synchronization volume} becomes a central systems constraint for scalable pretraining~\cite{narayanan2021efficient}.

Existing work mitigates distributed-training overhead from multiple angles.
System-level techniques such as sharding~\cite{rajbhandari2021zero} and generic compression~\cite{tang20211,karimireddy2019error} reduce the \emph{effective} synchronization cost, but largely operate independently of the \emph{matrix representation} being synchronized.
Complementary to these system-level remedies, projection-based low-rank optimizers such as GaLore~\cite{zhao2024galore} maintain low-dimensional optimizer states for memory efficiency, yet remain suboptimal for communication-limited training: with one-sided projection, synchronizing an $m\times n$ matrix gradient still requires transmitting an $O(rn)$ (or $O(mr)$) object, and SVD-style refresh can dominate the \emph{peak} communicated bytes.
Most recently, concurrent work such as GreedyLore~\cite{chen2025greedy} augments GaLore-style low-rank compression with error feedback and improved guarantees; however, it inherits the GaLore-type one-sided communication object (scaling with a matrix dimension) and does not achieve the $O(r^2)$ shared-core synchronization.

Motivated by a communication-first view of data-parallel training, we focus on the gradient as the synchronization object and measure efficiency by the transmitted bytes required to reach a target loss (Figure~\ref{fig:bytes_to_loss}).
For a matrix gradient $G\in\mathbb{R}^{m\times n}$, dense synchronization scales as $O(mn)$ per step, while one-sided low-rank methods still communicate an $O(rn)$ (or $O(mr)$) factor and remain expensive when one dimension is large (Table~\ref{tab:comm_complexity}). 
A natural way to break this scaling is to maintain two low-rank projection bases, $U\in\mathbb{R}^{m\times r}$ and $V\in\mathbb{R}^{n\times r}$, and synchronize only the core $C=U^\top G V\in\mathbb{R}^{r\times r}$, reducing the dominant payload to $O(r^2)$ while storing Adam moments in the same core space.
Making this approach practical hinges on two issues: (i) refresh steps must update  base without falling back to full-gradient synchronization, otherwise peak communicated bytes can dominate; and (ii) communication-heavy layers should be accounted for explicitly, notably embeddings, which are often kept dense in prior low-rank treatments but can occupy a non-trivial fraction of synchronization volume (Figure~\ref{fig:embedding}).

\begin{table}[t]
  \caption{Communication objects and scaling laws for synchronizing a matrix gradient $G\in\mathbb{R}^{m\times n}$.
  The dominant communicated payload changes from $O(mn)$ (dense) or $O(rn)$ (one-sided) to $O(r^2)$ (two-sided core).}
  \label{tab:comm_complexity}
  \begin{center}
    \begin{small}
      \begin{sc}
        \begin{tabular}{lccc}
          \toprule
          Method & Synchronized object & Size & Scaling \\
          \midrule
          AdamW & $G$ & $mn$ & $O(mn)$ \\
          LORA & $G_A,G_B(W^{'}=W+AB)$ & $rm+rn$ & $O(r(m+n))$ \\   
          One-sided  & $C=U^\top  G$ & $rn$ (or $mr$) & $O(rn)$ \\
          TSR  & $C=U^\top G V$ & $r^2$ & $O(r^2)$ \\
          \bottomrule
        \end{tabular}
      \end{sc}
    \end{small}
  \end{center}
  \vskip -0.1in
\end{table}

\begin{table}[th]
  \caption{The number of parameters in the Embedding and Linear layer weights, as well as in the optimizer state. Assume $W\in\mathbb{R}^{m\times n}$, rank $r$, Embedding rank $r_e$ and vocabulary size $\mathcal{V}$.}
  \label{tab:mem_complexity}
  \begin{center}
    \begin{small}
      \begin{sc}
        \begin{tabular}{lcc}
          \toprule
          Method & Weights &Optimizer State \\
          \midrule
          Embedding Layer&  &  \\
          \midrule
          Adam & $\mathcal{V}\times m$ & $\mathcal{V}\times m + 2\mathcal{V}\times m $ \\
          LoRA & $\mathcal{V}\times m$ &  $\mathcal{V}\times m + 2\mathcal{V}\times m $\\   
          One-sided  & $\mathcal{V}\times m $& $\mathcal{V}\times m+2\mathcal{V}\times m$ \\
          TSR & $\mathcal{V}\times m$ & $\mathcal{V}\times r_e+ r_e\times m+2r_e^2$ \\
          \midrule
          Linear Layers&  & \\
          \midrule
          Adam & $mn$ & $2mn$ \\
          LORA  & $mn+rm+rn$&$2mr + 2nr$ \\   
          One-sided  & $mn$ &$mr+2nr$ \\
          TSR  & $mn$& $mr+nr+2r^2$ \\
          \bottomrule
        \end{tabular}
      \end{sc}
    \end{small}
  \end{center}
  \vskip -0.2in
\end{table}

To address these two issues, we propose \textbf{TSR}, a two-sided low-rank communication mechanism for data-parallel training; in this paper we instantiate TSR with Adam-family updates, yielding \textbf{TSR-Adam}(Figure~\ref{fig:overview}).
At each step, TSR-Adam synchronizes only the core $C=U^\top G V$ and performs Adam moment updates in this same core space, reducing the dominant per-step communication from $O(mn)$ (dense) or $O(rn)$/$O(mr)$ (one-sided) to $O(r^2)$, while retaining low-dimensional moment states and thus the memory benefits of low-rank training.
To ensure subspace maintenance does not dominate \emph{peak} bytes, TSR-Adam refreshes the two projection bases via a randomized SVD procedure that communicates only low-dimensional sketches, avoiding full-gradient synchronization.
We further apply the same low-rank communication principle to \emph{embedding} gradients with embedding-specific ranks and refresh schedules, yielding additional communication and memory savings over keeping embeddings dense.
We evaluate TSR-Adam on pretraining from 60M to 1B model scales and on GLUE fine-tuning benchmarks, and provide a convergence analysis for the resulting update. Overall, our contributions are summarized as follows:

\begin{itemize}
    \item \textbf{Communication-first rethinking of low-rank optimizers.}
    We demonstrate that projection-based low-rank optimizers, originally developed for \emph{memory efficiency}, can be reframed as a communication mechanism. We optimize them explicitly for bandwidth-limited data-parallel training, surpassing the limitations of prior one-sided designs.

    \item \textbf{Two-sided $O(r^2)$ core synchronization with preserved memory benefits.}
    We introduce \textbf{TSR-Adam}, which synchronizes a compact core $C=U^\top G V\in\mathbb{R}^{r\times r}$ for each matrix gradient, reducing the dominant per-step payload from $O(mn)$ (dense) and $O(rn)$ (one-sided) to $O(r^2)$, while keeping Adam moment states in the same low-dimensional core space.

    \item \textbf{Practical peak-byte reduction and embedding-aware design.}
    We develop a randomized SVD-based refresh that avoids full-gradient synchronization to control \emph{peak} communication on refresh steps, and we extend low-rank communication to \emph{embedding} gradients with embedding-specific ranks and refresh schedules, yielding additional communication and memory savings over keeping embeddings dense.

    \item \textbf{Empirical gains and theory.}
    Across pretraining from 60M to 1B model scales, TSR-Adam reduces average communicated bytes per step by $13\times$ and on GLUE fine-tuning it reduces communication by $25\times$, with comparable performance; we further provide a theoretical stationarity guarantee for the resulting update.


\end{itemize}

\begin{figure*}[th]
  \begin{center}
    \centerline{\includegraphics[width=\textwidth]{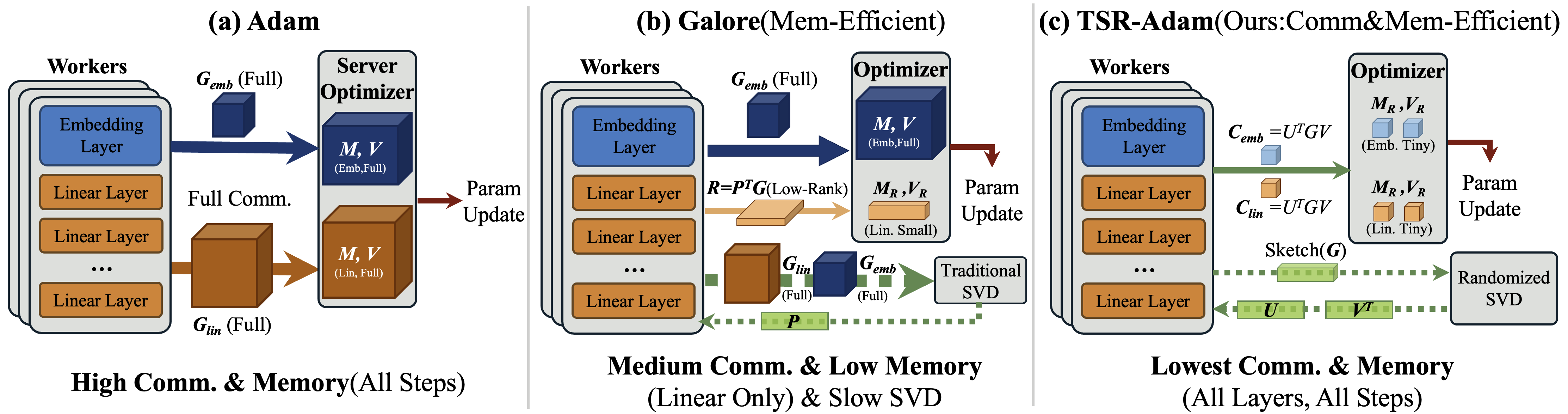}}
\caption{
  \textbf{Comparison of Communication Mechanisms.}
  Visualizing synchronized objects (\textbf{block volume}) and bandwidth usage (\textbf{arrow width}).
  \textbf{(a) Adam} transmits dense gradients ($O(mn)$).
  \textbf{(b) GaLore} compresses linear layers ($O(rn)$) but leaves embeddings dense (large blue cube) and uses heavy SVD refresh.
  \textbf{(c) TSR-Adam} synchronizes tiny $r \times r$ cores ($O(r^2)$) across \emph{all} layers and uses lightweight sketches (green bars) for refresh, achieving the lowest communication footprint.
}

    \label{fig:overview}
  \end{center}
  \vskip -0.2in
\end{figure*}

\section{Related Work}
\label{sec:related}

\paragraph{Communication bottlenecks in large-scale LLM pretraining.}
As foundation models continue to scale~\cite{brown2020language,bommasani2021opportunities,touvron2023llama},
pretraining increasingly relies on data-parallel (often hybrid-parallel) distributed optimization, where dense gradient synchronization can become a dominant cost once training spans many devices and nodes~\cite{narayanan2021efficient,rasley2020deepspeed,shoeybi2019megatron,jiang2024megascale}.
Prior work mitigates distributed training overheads through multiple routes:
(i) \emph{state/gradient sharding} such as ZeRO-style partitioning that removes redundancy across workers~\cite{rajbhandari2020zero,rajbhandari2021zero,zhao2023pytorch};
(ii) \emph{quantization/compression} that transmits low-bit updates with convergence guarantees (e.g., QSGD)~\cite{alistarh2017qsgd} and extends to adaptive optimizers (e.g., 1-bit Adam)~\cite{tang20211};
(iii) \emph{sparsification} that communicates only salient entries with error feedback/momentum correction (e.g., DGC)~\cite{lin2017deep,karimireddy2019error};
and (iv) \emph{structured compression} such as low-rank matrix-factor communication, where gradients are approximated by low-rank factors to reduce payload~\cite{vogels2019powersgd}.
Recent work also explores GaLore-style low-rank compression with improved error control and convergence analysis (e.g., GreedyLore)~\cite{chen2025greedy}.
Importantly, most existing low-rank compressors synchronize factors or subspaces whose size still scales with at least one large matrix dimension (commonly $O(r(m+n))$ or $O(rn)$), and  subspace-refresh steps can induce pronounced \emph{peak} communication.

\paragraph{Low-rank optimization for memory-efficient training.}
In addition to bandwidth, LLM training is often constrained by memory due to optimizer states and activation storage.
Low-rank methods reduce memory by restricting adaptation or optimization to compact subspaces.
For parameter-efficient fine-tuning, LoRA~\cite{hu2022lora} injects trainable low-rank adapters while keeping pretrained weights frozen, with widely-used variants improving rank allocation, quantization, and stability (e.g., AdaLoRA, QLoRA, DoRA)~\cite{zhang2023adalora,dettmers2023qlora,liu2024dora}.
For \emph{pretraining} and full-parameter learning, projection-based low-rank optimizers such as GaLore~\cite{zhao2024galore} maintain full-rank weights but project gradients into a learned low-rank subspace to reduce optimizer-state memory, with subsequent extensions improving scalability and refresh/staleness handling (e.g., GaLore 2, Tensor-GaLore; Online Subspace Descent; ``Breaking the Frozen Subspace'')~\cite{su2025galore,george2024tensor,liang2024memory,zhang2025breaking}.
Overall, much of the low-rank optimizer literature is primarily driven by \emph{memory efficiency}, whereas optimizing for \emph{communication-limited} pretraining---especially achieving dimension-independent payload scaling and accounting for peak communicated bytes---remains comparatively under-explored.

\section{Method}
\label{sec:method}

\subsection{Background: AdamW}
\label{subsec:background}

We consider data-parallel distributed training with $N$ workers.
Let $f(w)$ denote the training objective and $w$ the model parameters.
At step $t$, worker $i$ computes a stochastic gradient $g_{t,i}=\nabla f_i(w_t)$ over its local mini-batch.
In standard data parallelism, the optimizer update is applied using the \emph{globally aggregated} gradient
\begin{equation}
\bar{g}_t \;=\; \frac{1}{N}\sum_{i=1}^{N} g_{t,i},
\label{eq:dp_grad_avg}
\end{equation}
which is typically obtained via an all-reduce collective.
For clarity, we first review the AdamW update under dense synchronization and then identify the tensors that
must be communicated.

\paragraph{Dense AdamW update.}
AdamW maintains first- and second-moment estimates $(m_t, v_t)$ and applies decoupled weight decay.
Given the aggregated gradient $\bar{g}_t$, the moment updates are
\begin{align}
m_t &= \beta_1 m_{t-1} + (1-\beta_1)\,\bar{g}_t, \label{eq:adam_m_dense}\\
v_t &= \beta_2 v_{t-1} + (1-\beta_2)\,\big(\bar{g}_t \circ \bar{g}_t\big), \label{eq:adam_v_dense}
\end{align}
followed by bias correction $\hat{m}_t=m_t/(1-\beta_1^t)$ and $\hat{v}_t=v_t/(1-\beta_2^t)$.
The parameter update is
\begin{equation}
w_{t+1}
\;=\;
w_t - \eta \left( \frac{\hat{m}_t}{\sqrt{\hat{v}_t}+\epsilon} + \lambda\,w_t \right).
\label{eq:adamw_update_dense}
\end{equation}
In distributed data parallelism, the \emph{communication-critical} step is the construction of $\bar{g}_t$ in
\eqref{eq:dp_grad_avg}; once $\bar{g}_t$ is available, all remaining computations in
\eqref{eq:adam_m_dense}--\eqref{eq:adamw_update_dense} are local.

\paragraph{Matrix-shaped parameters.}
Our method targets matrix-shaped parameter blocks, which constitute the dominant portion of trainable weights
in transformers (e.g., linear and embedding matrices).
For a matrix parameter block $W^{(\ell)}\in\mathbb{R}^{m_\ell\times n_\ell}$, let
\begin{equation*}
G_{t,i}^{(\ell)} \;:=\; \nabla_{W^{(\ell)}} f_i(w_t) \in \mathbb{R}^{m_\ell\times n_\ell}
\end{equation*}
denote worker $i$'s local gradient, and
\begin{equation*}
\bar{G}_t^{(\ell)} \;=\; \frac{1}{N}\sum_{i=1}^{N} G_{t,i}^{(\ell)}
\end{equation*}
the synchronized gradient used by dense AdamW.
Dense synchronization communicates an $m_\ell\times n_\ell$ tensor per step for each matrix block $\ell$,
which scales with the full parameter size and can become the dominant bottleneck when interconnect bandwidth
is limited.

\subsection{Problem Formulation}
\label{subsec:formulation}

We consider data-parallel training with $N$ workers and focus on matrix-shaped parameter blocks
$\{W^{(\ell)}\}_{\ell\in\mathcal{L}_{\text{mat}}}$.
At step $t$, each worker $i$ computes a local matrix gradient $G_{t,i}^{(\ell)}$ and dense synchronization
forms $\bar{G}_t^{(\ell)}=\frac{1}{N}\sum_{i=1}^{N} G_{t,i}^{(\ell)}$, which is then used by AdamW.
From a communication perspective, the key question is: \emph{what tensors must be synchronized across workers at each step}?

For a given method, let $\mathcal{S}_t^{(\ell)}$ denote the set of tensors synchronized for layer $\ell$ at step $t$.
For dense AdamW, $\mathcal{S}_t^{(\ell)}=\{\bar{G}_t^{(\ell)}\}$.
Low-rank communication replaces $\bar{G}_t^{(\ell)}$ by compact representations, and may additionally synchronize
auxiliary tensors on subspace-refresh steps.

Let $b_{\text{dtype}}$ be the number of bytes per element of the communicated dtype (e.g., 2 for bf16/fp16, 4 for fp32),
and let $|\mathcal{S}_t^{(\ell)}|$ denote the total number of scalar entries across all tensors in the set $\mathcal{S}_t^{(\ell)}$.
We define the step-wise communicated bytes at step $t$ as
\begin{equation*}
\mathcal{B}_t \;:=\; \sum_{\ell\in\mathcal{L}_{\text{mat}}} b_{\text{dtype}}\cdot\big|\mathcal{S}_t^{(\ell)}\big|.
\end{equation*}

Over $T$ optimization steps, we further report the average communicated bytes per step,
$\mathrm{Bytes/Step}:=\frac{1}{T}\sum_{t=1}^{T} \mathcal{B}_t$, the peak communicated bytes per step,
$\mathrm{PeakBytes}:=\max_{1\le t\le T} \mathcal{B}_t$, and the cumulative communicated bytes up to step $t$,
$\mathrm{CumulativeBytes}(t):=\sum_{\tau=1}^{t} \mathcal{B}_\tau$.

\newcommand{\algc}[1]{\hfill{\footnotesize$\triangleright$ #1}}

\begin{algorithm}[tb]
\caption{TSR-Adam for a matrix block $W^{(\ell)}$ (use $(r_{\text{emb}},K_{\text{emb}})$ for embeddings; $(r,K)$ otherwise)}
\label{alg:tsr-adam}
\centering
\begin{minipage}{0.97\columnwidth}
\small
\begin{algorithmic}
\STATE {\bfseries Input:} rank $r_\ell$, refresh interval $K_\ell$, oversampling $p$,
AdamW hyperparams $(\eta,\lambda,\beta_1,\beta_2,\epsilon)$
\STATE {\bfseries State:} bases $(U,V)$, core moments $(m,v)$, step $t\leftarrow 1$
\STATE Set $r\leftarrow r_\ell$, $K\leftarrow K_\ell$, $k\leftarrow r+p$.
Initialize $m\leftarrow 0$, $v\leftarrow 0$, and $(U,V)$ by one refresh.
Let $\AR(\cdot)$ denote all-reduce averaging across workers.
\STATE (Algorithm is for a fixed layer $\ell$; we omit $(\ell)$ in variables for brevity.)

\REPEAT
  \STATE Each worker $i$ computes local gradient $G_{t,i}$

  \IF{$t\bmod K = 0$}
    \STATE Sample shared $\Omega$ (shared RNG seed)
    \STATE $Y_{t,i}\leftarrow G_{t,i}\Omega$
    \STATE $Q_{t,i} \leftarrow \mathrm{orth}(Y_{t,i})$
    , \quad  $Y^{\mathrm{row}}_{t,i}\leftarrow G_{t,i}^\top Q_{t,i}$
    \STATE   $Q^{\mathrm{row}}_{t,i} \leftarrow \mathrm{orth}(Y^{\mathrm{row}}_{t,i})$
    \STATE $Y_{t,i}\leftarrow G_{t,i}Q^{\mathrm{row}}_{t,i}$
    \STATE $Q_{t,i} \leftarrow \mathrm{orth}(Y_{t,i})$
    \STATE $B_{t,i}\leftarrow Q_{t,i}^\top G_{t,i},\quad \bar{B}_t \leftarrow \AR(B_{t,i})$ \algc{$\bar{B}_t$}
    \STATE $\bar{Q}_t \leftarrow \AR(Q_{t,i}),\quad\bar{B}_t=\widetilde{U}\Sigma\widetilde{V}^\top$ \algc{$\bar{Q}_t$}
    \STATE $U\leftarrow \bar{Q}\,\widetilde{U}_{[:,1:r]},\quad V\leftarrow \widetilde{V}_{[:,1:r]}$
  \ENDIF

  \STATE $C_{t,i}\leftarrow U^\top G_{t,i} V,\quad \bar{C}_t \leftarrow \AR(C_{t,i})$ \algc{$\bar{C}_t$}
  \STATE $m \leftarrow \beta_1 m + (1-\beta_1)\bar{C}_t$
  \STATE $v \leftarrow \beta_2 v + (1-\beta_2)\big(\bar{C}_t \circ \bar{C}_t\big)$
  \STATE $\hat{m}\leftarrow m/(1-\beta_1^{t}),\quad \hat{v}\leftarrow v/(1-\beta_2^{t})$
  \STATE $D \leftarrow \hat{m}\oslash(\sqrt{\hat{v}}+\epsilon)$
  \STATE $\Delta W_t \leftarrow U D V^\top$
  \STATE $W \leftarrow W - \eta\left(\Delta W_t + \lambda W\right)$
  \STATE $t \leftarrow t+1$
\UNTIL{convergence criteria met}
\end{algorithmic}
\end{minipage}
\end{algorithm}

\subsection{TSR-Adam: Two-Sided Low-Rank Core Synchronization}
\label{subsec:tsr}

We compress and synchronize matrix gradients through a two-sided low-rank core representation.
For each matrix-shaped parameter block $W^{(\ell)}\in\mathbb{R}^{m_\ell\times n_\ell}$, TSR-Adam maintains
orthonormal bases $U_t^{(\ell)}\in\mathbb{R}^{m_\ell\times r_\ell}$ and
$V_t^{(\ell)}\in\mathbb{R}^{n_\ell\times r_\ell}$, with $(U_t^{(\ell)})^\top U_t^{(\ell)}=I$ and
$(V_t^{(\ell)})^\top V_t^{(\ell)}=I$.
At step $t$, worker $i$ computes the local matrix gradient $G_{t,i}^{(\ell)}$ and forms the local core
\begin{equation*}
C_{t,i}^{(\ell)} \;:=\; (U_t^{(\ell)})^\top G_{t,i}^{(\ell)} V_t^{(\ell)} \in \mathbb{R}^{r_\ell\times r_\ell}.
\end{equation*}
We then synchronize only the core via all-reduce,
$\bar{C}_t^{(\ell)}=\frac{1}{N}\sum_{i=1}^{N} C_{t,i}^{(\ell)}$,
and reconstruct the gradient locally as
\begin{equation}
\widehat{G}_t^{(\ell)} \;:=\; U_t^{(\ell)}\,\bar{C}_t^{(\ell)}\,(V_t^{(\ell)})^\top .
\label{eq:recon_grad}
\end{equation}
Compared with dense synchronization of $\bar{G}_t^{(\ell)}\in\mathbb{R}^{m_\ell\times n_\ell}$, this reduces the
\emph{per-step} synchronized payload for layer $\ell$ from $O(m_\ell n_\ell)$ elements to $O(r_\ell^2)$ elements
on non-refresh steps.

\subsection{AdamW Updates in Core Space and Reconstruction}
\label{subsec:adamw_tsr}

TSR-Adam performs AdamW updates using the synchronized two-sided core as the effective gradient.
For each matrix block $W^{(\ell)}$, we maintain AdamW first- and second-moment states
$m_t^{(\ell)}, v_t^{(\ell)} \in \mathbb{R}^{r_\ell\times r_\ell}$ in the low-dimensional core space.
At step $t$, after synchronizing the core $\bar{C}_t^{(\ell)}=\frac{1}{N}\sum_{i=1}^{N}(U_t^{(\ell)})^\top G_{t,i}^{(\ell)} V_t^{(\ell)}$,
we update
\begin{align*}
m_t^{(\ell)} &= \beta_1 m_{t-1}^{(\ell)} + (1-\beta_1)\,\bar{C}_t^{(\ell)}, \\
v_t^{(\ell)} &= \beta_2 v_{t-1}^{(\ell)} + (1-\beta_2)\,\big(\bar{C}_t^{(\ell)} \circ \bar{C}_t^{(\ell)}\big), 
\end{align*}
where $\circ$ denotes the Hadamard (element-wise) product.
With bias correction $\hat{m}_t^{(\ell)}=m_t^{(\ell)}/(1-\beta_1^t)$ and $\hat{v}_t^{(\ell)}=v_t^{(\ell)}/(1-\beta_2^t)$,
the normalized update in core space is
\begin{equation*}
D_t^{(\ell)} \;=\; \hat{m}_t^{(\ell)} \oslash \big(\sqrt{\hat{v}_t^{(\ell)}}+\epsilon\big),
\end{equation*}
where $\oslash$ is element-wise division.
We then lift the update back to the original parameter space via
\begin{equation*}
\Delta W_t^{(\ell)} \;=\; U_t^{(\ell)}\, D_t^{(\ell)}\, (V_t^{(\ell)})^\top,
\end{equation*}
and apply the standard AdamW rule
\begin{equation*}
W_{t+1}^{(\ell)}
\;=\;
W_t^{(\ell)} - \eta\left(
\Delta W_t^{(\ell)} + \lambda\,W_t^{(\ell)}
\right).
\end{equation*}
All non-matrix parameters (e.g., biases and normalization parameters) are synchronized and updated in dense form.

\subsection{Randomized Refresh of Two-Sided Bases and Peak Communication}
\label{subsec:rsvd_refresh}

The bases $(U_t^{(\ell)},V_t^{(\ell)})$ must track the evolving gradient subspaces; stale bases can degrade the
quality of the reconstructed update direction in Eq.~\eqref{eq:recon_grad}.

A naive refresh strategy would first synchronize the full dense gradient $\bar{G}_t^{(\ell)}$ and then compute an exact SVD,
incurring both high compute cost and a large \emph{peak} communication payload on refresh steps.
Instead, we refresh $(U_t^{(\ell)},V_t^{(\ell)})$ using a sketch-based randomized SVD (rSVD) procedure that avoids
full-gradient synchronization and only requires communicating low-dimensional sketches.

Let $k_\ell=r_\ell+p$ with a small oversampling $p$.
On refresh steps (every $K_\ell$ steps), we draw a shared Gaussian matrix
$\Omega^{(\ell)}\in\mathbb{R}^{n_\ell\times k_\ell}$ (e.g., via RNG seed).
Each worker $i$ forms a range sketch
$Y_{t,i}^{(\ell)} = G_{t,i}^{(\ell)}\Omega^{(\ell)}$ then computes an orthonormal basis $Q_{t,i}^{(\ell)}=\mathrm{orth}(Y_{t,i}^{(\ell)})$ (implemented by thin QR).

To improve the approximation when the spectrum decays slowly, we optionally apply $q$ steps of power iteration.
In Algorithm~\ref{alg:tsr-adam} we show the case $q=1$ for concreteness; larger $q$ follows by repeating the alternating multiplications.
Concretely, for $j=1,\ldots,q$, we compute
\begin{align*}
{Y}_{t,i}^{\mathrm{row},(\ell)} \;=\; \! (G_{t,i}^{(\ell)})^\top Q_{t,i}^{(\ell)},
\quad
{Q}_{t,i}^{\mathrm{row},(\ell)}=\mathrm{orth}\!\left({Y}_{t,i}^{\mathrm{row},(\ell)}\right), \\
{Y}_{t,i}^{(\ell)} \;=\; \!G_{t,i}^{(\ell)} {Q}_{t,i}^{\mathrm{row},(\ell)},
\quad
{Q}_{t,i}^{(\ell)}=\mathrm{orth}\!\left({Y}_{t,i}^{(\ell)}\right),
\end{align*}
which matches the standard alternating multiplications by $G^{(\ell)}$ and $(G^{(\ell)})^\top$ in randomized SVD.

Given the final $Q_{t,i}^{(\ell)}$, we form the reduced matrix

\begin{gather*}
\bar{B}_t^{(\ell)} \;:=\; (Q_t^{(\ell)})^\top \bar{G}_t^{(\ell)}
\;=\; \AR\!\left( (Q_{t,i}^{(\ell)})^\top G_{t,i}^{(\ell)} \right), \\
\bar{Q}_t^{(\ell)} \;=\; \AR\! \left(Q_{t,i}^{(\ell)}\right)   
\end{gather*}

and compute a small SVD $\bar{B}_t^{(\ell)} = \widetilde{U}\Sigma \widetilde{V}^\top$.
We then refresh the two-sided bases as
\begin{equation*}
U_t^{(\ell)} \leftarrow \bar{Q}_t^{(\ell)}\,\widetilde{U}_{[:,1:r_\ell]},
\qquad
V_t^{(\ell)} \leftarrow \widetilde{V}_{[:,1:r_\ell]} .
\end{equation*}

Accordingly, on non-refresh steps TSR-Adam synchronizes only $\mathcal{S}_t^{(\ell)}=\{\bar{C}_t^{(\ell)}\}$.
On refresh steps, $\mathcal{S}_t^{(\ell)}$ additionally includes the communicated sketches
$\bar{Q}_t^{(\ell)}$ and $\bar{B}_t^{(\ell)}$,
which determines the peak communicated bytes.

Importantly, refresh communicates only sketches of size $m_\ell k_\ell$ or $n_\ell k_\ell$ (repeated a small number
of times when $q>0$), rather than the full $m_\ell n_\ell$ gradient, thereby substantially reducing the peak communicated
bytes per step captured by $\mathrm{PeakBytes}$.

\begin{table}[t]
  \caption{Main Results Summary. We report final loss, bytes per step, peak communication volume, memory, and the average update time per subspace refresh interval (including relative communication reduction factors).}
  \label{tab:main}
  \begin{center}
  \setlength{\tabcolsep}{2.3pt}
    \begin{small}
      \begin{sc}
        \begin{tabular}{lcccccccc}
          \toprule
          Scale & Method & Rank & $K$ & Final loss$\downarrow$ & Bytes/step$\downarrow$ & Peak Bytes &  Memory& Update Time \\
          \midrule
          60M  & AdamW & 512 & -- & 3.53 & 0.17G & 0.17G & 0.28G&0.42s \\
          60M  & GaLore & 128 & 200 & 3.55 & 0.10G & 0.14G & 0.21G&0.43s  \\
          \rowcolor{blue!10} 60M  & TSR & 256(64) & 100 & 3.61 & 0.020G & 0.10G & 0.17G&0.45s  \\
          \midrule
          130M & AdamW & 768 & -- & 3.22 & 0.44G & 0.44G &0.71G&0.41s\\
          130M & GaLore & 256 & 200 & 3.23 & 0.21G & 0.36G & 0.51G&0.59s\\
          \rowcolor{blue!10} 130M & TSR & 384(96) & 100 & 3.36 & 0.058G & 0.31G & 0.45G&0.41s\\
          \midrule
          350M & AdamW & 1024 & -- & 2.93 & 1.34G & 1.34G & 2.08G&0.50s\\
          350M & GaLore & 256 & 200 & 2.99 & 0.44G & 0.98G & 1.26G&0.83s\\
          \rowcolor{blue!10} 350M & TSR & 384(128) & 100 & 3.16 & 0.11G & 0.79G & 1.19G&0.69s\\
          \midrule
          1B   & AdamW & 2048 & -- & 2.74 & 5.09G & 5.09G &7.77G&3.19s\\
          1B   & GaLore & 512 & 200 & 2.75 & 1.48G & 3.63G & 4.5G&4.01s\\
          \rowcolor{blue!10} 1B   & TSR & 512(256) & 100 & 3.08 & 0.21G &2.05G & 3.81G&3.19s\\
          \bottomrule
        \end{tabular}
      \end{sc}
    \end{small}
  \end{center}
  \vskip -0.2in
\end{table}

\subsection{Embedding-Specific Low-Rank Communication}
\label{subsec:emb}

Embedding matrices can contribute a substantial fraction of communicated gradient volume, especially in smaller models.
We therefore apply the same two-sided core synchronization to the embedding matrix
$E\in\mathbb{R}^{|\mathcal{V}|\times d}$, but use a separate rank and refresh schedule,
$(r_{\text{emb}},K_{\text{emb}})$, from other linear layers. Concretely, we maintain
$U_t^{(\text{emb})}\in\mathbb{R}^{|\mathcal{V}|\times r_{\text{emb}}}$ and
$V_t^{(\text{emb})}\in\mathbb{R}^{d\times r_{\text{emb}}}$ and synchronize the embedding core
$C_{t,i}^{(\text{emb})}=(U_t^{(\text{emb})})^\top G_{t,i}^{(\text{emb})} V_t^{(\text{emb})}$.
This decoupling allows TSR-Adam to tailor communication and refresh costs to the distinct structure of embeddings
without constraining the choices for other matrix blocks.

\subsection{Convergence Analysis}
\label{subsec:convergence}

Providing convergence guarantees for fully adaptive Adam-family methods is notoriously delicate.
Following the common practice in low-rank/subspace optimization (e.g.,~\cite{he2024subspace,zhang2025breaking}), we therefore
establish a stationarity guarantee for a \emph{momentum} version of TSR that captures the core effect of
two-sided projection and refresh; all proofs are deferred to Appendix~\ref{app_sec:proofs}.
Concretely, we analyze the update
$\widetilde m_t := U_t m_t V_t^\top$ and $w_{t+1}=w_t-\eta\,\widetilde m_t$.

\paragraph{Assumption (Smoothness).}
The objective $f$ is $L$-smooth,
i.e., $f(y)\le f(x)+\langle \nabla f(x),y-x\rangle+\frac{L}{2}\|y-x\|^2$ for all $x,y$.

\paragraph{Assumption (Projected Stochasticity and Approximation).}
Let $P_t := U_t U_t^\top \nabla f(w_t) V_t V_t^\top$ be the projection of the true gradient onto the current subspace.
We define the subspace approximation error as $\Delta_t := \|P_t - \nabla f(w_t)\|^2$.
Furthermore, we assume the synchronized core provides an unbiased estimate of the projected gradient, i.e., $\mathbb{E}[U_t \bar{C}_t V_t^\top] = P_t$, with bounded variance $\mathbb{E}\|U_t \bar{C}_t V_t^\top - P_t\|^2 \le \sigma_t^2$.

\begin{theorem}[Stationarity bound]
\label{thm:tsr_main}
Suppose the above assumptions hold.
Set $\eta = \frac{1}{L\,T^{2/3}}$ and choose $\beta^2 = 1-\sqrt{40L\eta}=1-\frac{\sqrt{40}}{T^{1/3}}$
(so $1-\beta^2=\frac{\sqrt{40}}{T^{1/3}}$).
Define $\Delta f:=\mathbb{E}[f(w_0)-f^\star]$, let $\mathbb{I}_{\mathrm{refresh}}(t)\in\{0,1\}$ indicate whether step $t$ is a refresh step, and let
$R_t:=\left\|U_t m_{t-1}V_t^\top \allowbreak-\allowbreak U_{t-1} m_{t-1}V_{t-1}^\top\right\|^2$.
Then for any $T\ge 1$,
\begin{multline}
\frac{1}{T}\sum_{t=0}^{T-1}\mathbb{E}\|\nabla f(w_t)\|^2
\;\le\;
\frac{2L\,\Delta f}{T^{1/3}}
+\frac{4}{\sqrt{10}\,T^{2/3}}\mathbb{E}\|\widetilde m_0-\nabla f(w_0)\|^2
+\\
\frac{16}{T}\cdot
\frac{\frac{\sqrt{40}}{T^{1/3}}}{\Big(1+\sqrt{1-\frac{\sqrt{40}}{T^{1/3}}}\Big)^2}
\sum_{t=1}^{T-1}\sigma_t^2 +\frac{16}{T\Big(1+\sqrt{1-\frac{\sqrt{40}}{T^{1/3}}}\Big)^2}
\sum_{t=1}^{T-1}\Delta_t\\
+\left(\frac{4}{\sqrt{10}\,T^{2/3}}+\frac{8}{5\,T^{1/3}}\right)
\sum_{t=1}^{T-1}\mathbb{I}_{\mathrm{refresh}}(t)\,\mathbb{E}R_t.
\label{eq:main_stationarity}
\end{multline}
\end{theorem}

\begin{remark}[Interpretation of Convergence Terms]
The bound effectively decomposes into three parts:
\begin{itemize}
    \item \textbf{Stochastic Error:} Depends on the averaged projected variance $\bar{\sigma}^2:=\frac{1}{T}\sum_{t}\sigma_t^2$, damped by a factor proportional to $1-\beta$, yielding an overall $\tilde{O}(T^{-1/3})\cdot \bar{\sigma}^2$ contribution.
    \item \textbf{Approximation Floor:} Scales with $\bar{\Delta}:=\frac{1}{T}\sum_{t}\Delta_t$, forming an irreducible error floor. Notably, since gradients in large-scale training typically exhibit a low intrinsic dimension, $\bar{\Delta}$ remains naturally small, validating the effectiveness of our approach. However, if rank $r$ is chosen too small, $\bar{\Delta}$ will increase significantly, dominating the bound and leading to divergence.
    \item \textbf{Refresh Trade-off:} The refresh-mismatch term aggregates $\sum \mathbb{I}_{\mathrm{refresh}}(t)\mathbb{E}R_t$. This reflects a practical trade-off: larger $K$ reduces the number of refresh steps (lowering the sum count), whereas overly large $K$ may increase the subspace error $\Delta_t$ via staleness.
\end{itemize}
\end{remark}

\section{Experiments}
\label{sec:experiments}

\subsection{Experimental Setup}
\label{subsec:exp_setup}

We evaluate TSR on both pre-training and fine-tuning of LLMs. All experiments run on NVIDIA A100 GPUs and NVIDIA L40S GPUs.

\begin{figure}[t]
  \centering
  \begin{minipage}{0.32\columnwidth}
    \centering
    \includegraphics[width=\linewidth]{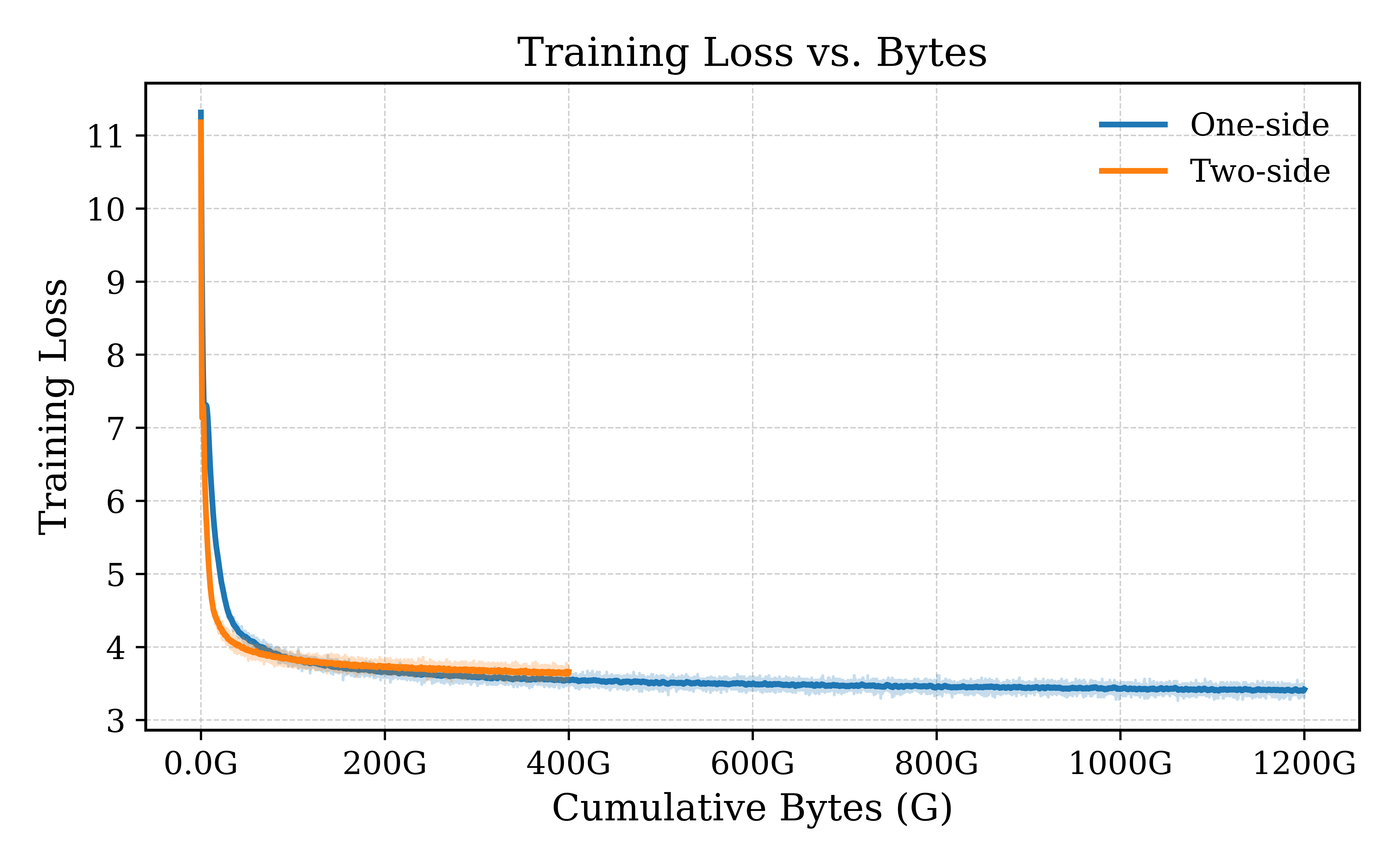}\\
    (a)
  \end{minipage}
  \hfill
  \begin{minipage}{0.32\columnwidth}
    \centering
    \includegraphics[width=\linewidth]{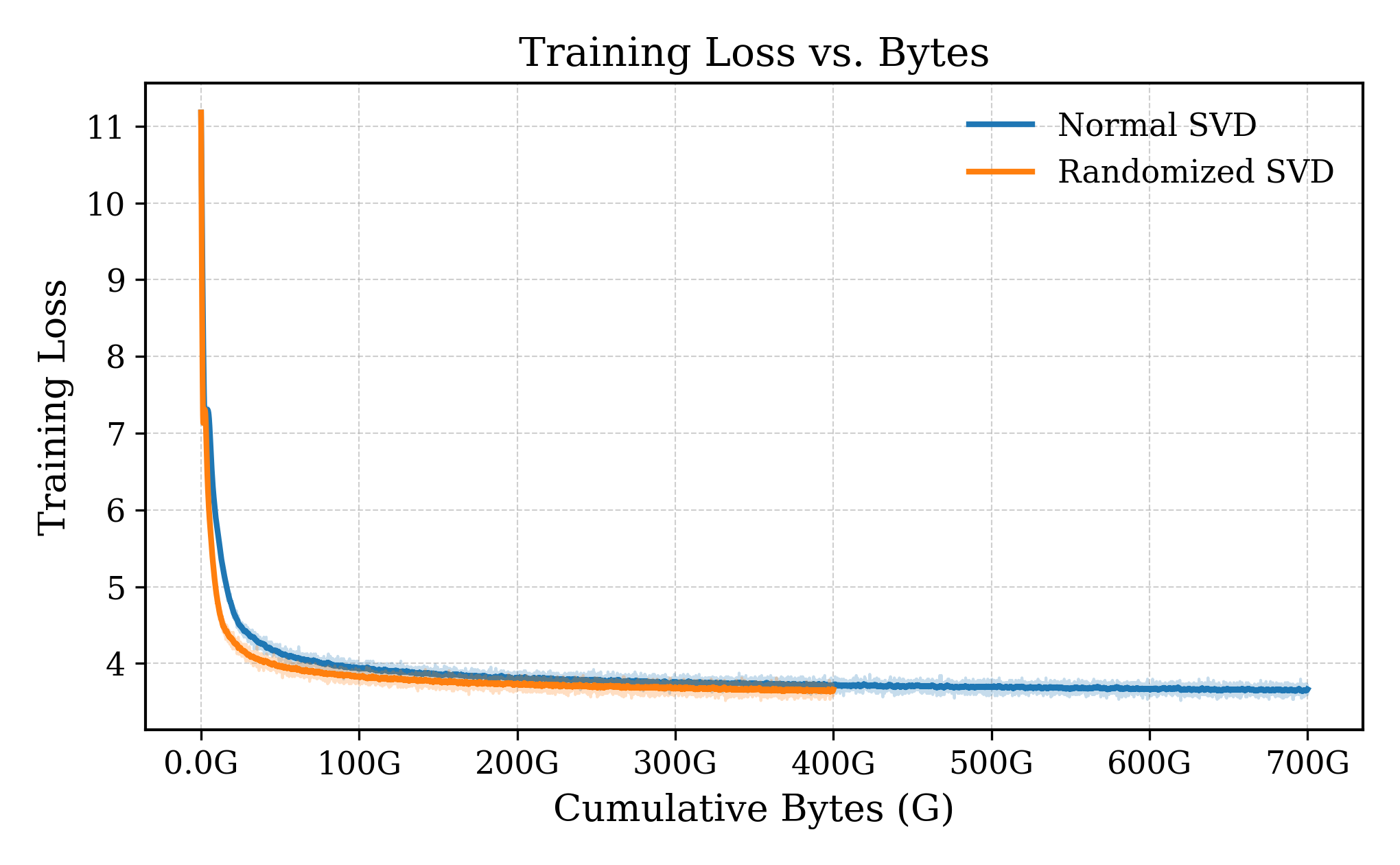}\\
    (b)
  \end{minipage}
  \hfill
  \begin{minipage}{0.32\columnwidth}
    \centering
    \includegraphics[width=\linewidth]{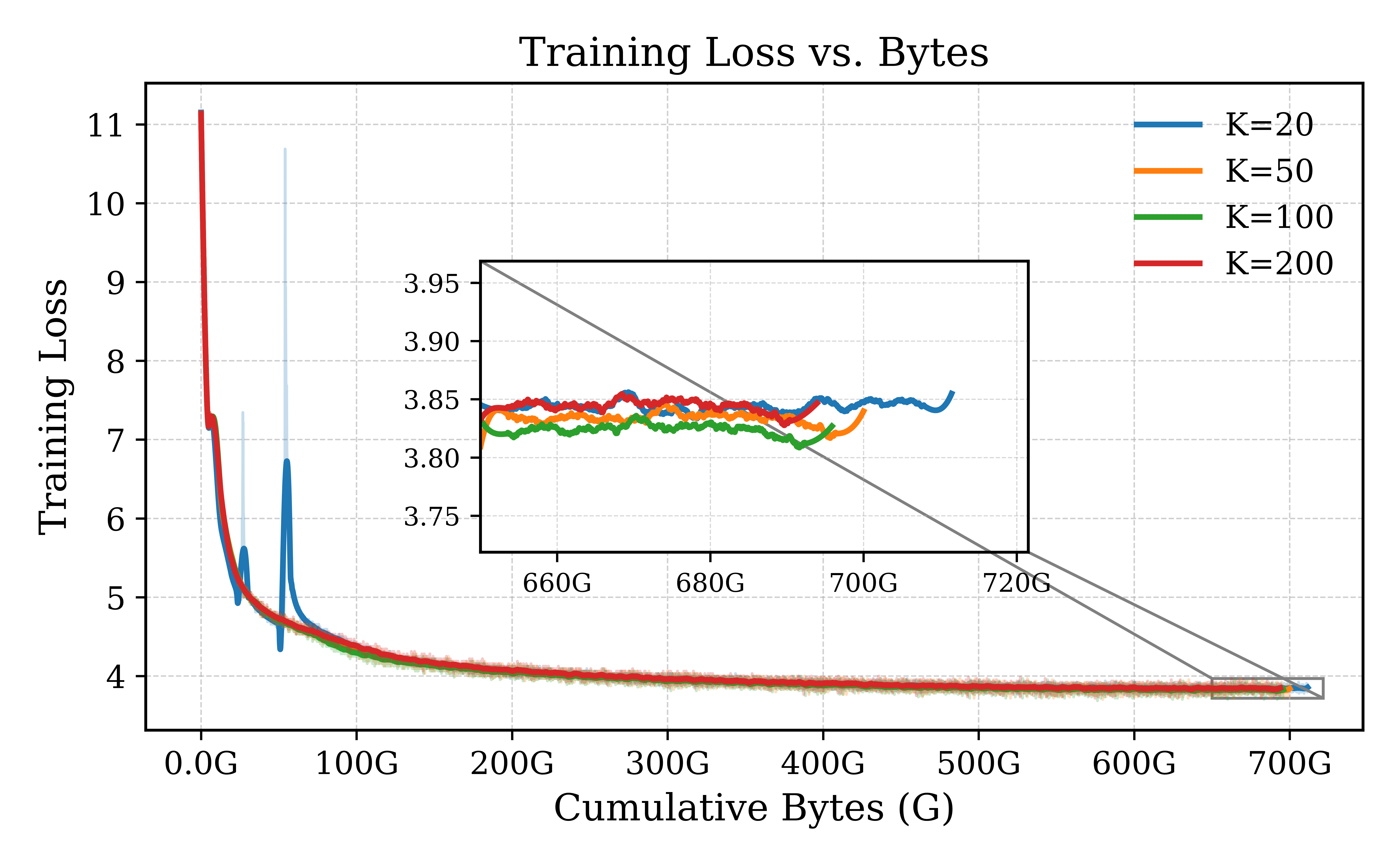}\\
    (c)
  \end{minipage}

  \caption{
    \textbf{Ablations.}
    We isolate the effects of (a) one-sided vs two-sided compression,
    (b) randomized SVD-style refresh,
    and (c) subspace refresh interval $K$ on the loss--communication trade-off.
  }
  \label{fig:ablation}
\end{figure}

\textbf{Pre-training on C4.} To evaluate its performance, we apply TSR to train LLaMA-based large language models~\cite{touvron2023llama} of four different scales—60M, 130M, 350M, and 1B—on the C4 dataset~\cite{raffel2020exploring}. The C4 dataset is a large, cleaned version of the Common Crawl web corpus, primarily designed for pre-training language models and learning word representations.

\textbf{Fine-tuning on GLUE tasks.} GLUE~\cite{wang2018glue} is a benchmark for
evaluating the performance of NLP models on a variety
of tasks, including sentiment analysis, question answering,
and textual entailment. We use GLUE
tasks to benchmark TSR against Adam and GaLore for communication
efficient fine-tuning.

\subsection{Main Results: Pretraining Communication Efficiency}
\label{subsec:pretrain_main}

\textbf{TSR improves the loss--communication trade-off across model scales.}
Figure~\ref{fig:pareto} and Table~\ref{tab:main} show that TSR-Adam achieves pretraining loss comparable to dense AdamW and GaLore on 60M/130M/350M models.
For LLaMA-1B, we report a representative run under limited training budget, and TSR still exhibits a highly competitive efficiency profile.
Across all scales, TSR substantially reduces communication and peak traffic: compared to AdamW, it cuts Bytes/Step by up to $24\times$, reduces PeakBytes by up to 59.72\%, and lowers the memory footprint by up to 50\%; averaged over scales, it achieves a $13\times$ reduction in Bytes/Step with consistent PeakBytes and memory savings.
In addition, TSR’s per-step update time is lower than GaLore and remains comparable to dense AdamW.

\textbf{TSR is more bytes-to-loss efficient under a fixed communication budget.}
On LLaMA-60M, we plot training loss against cumulative communicated bytes in Figure~\ref{fig:bytes_to_loss}.
TSR reaches comparable loss with markedly fewer communicated bytes, suggesting improved end-to-end communication efficiency beyond per-step compression. More results can be found in the Appendix~\ref{app:exp}.

\begin{figure}[t]
  \begin{center}
    \centerline{\includegraphics[width=0.7\columnwidth]{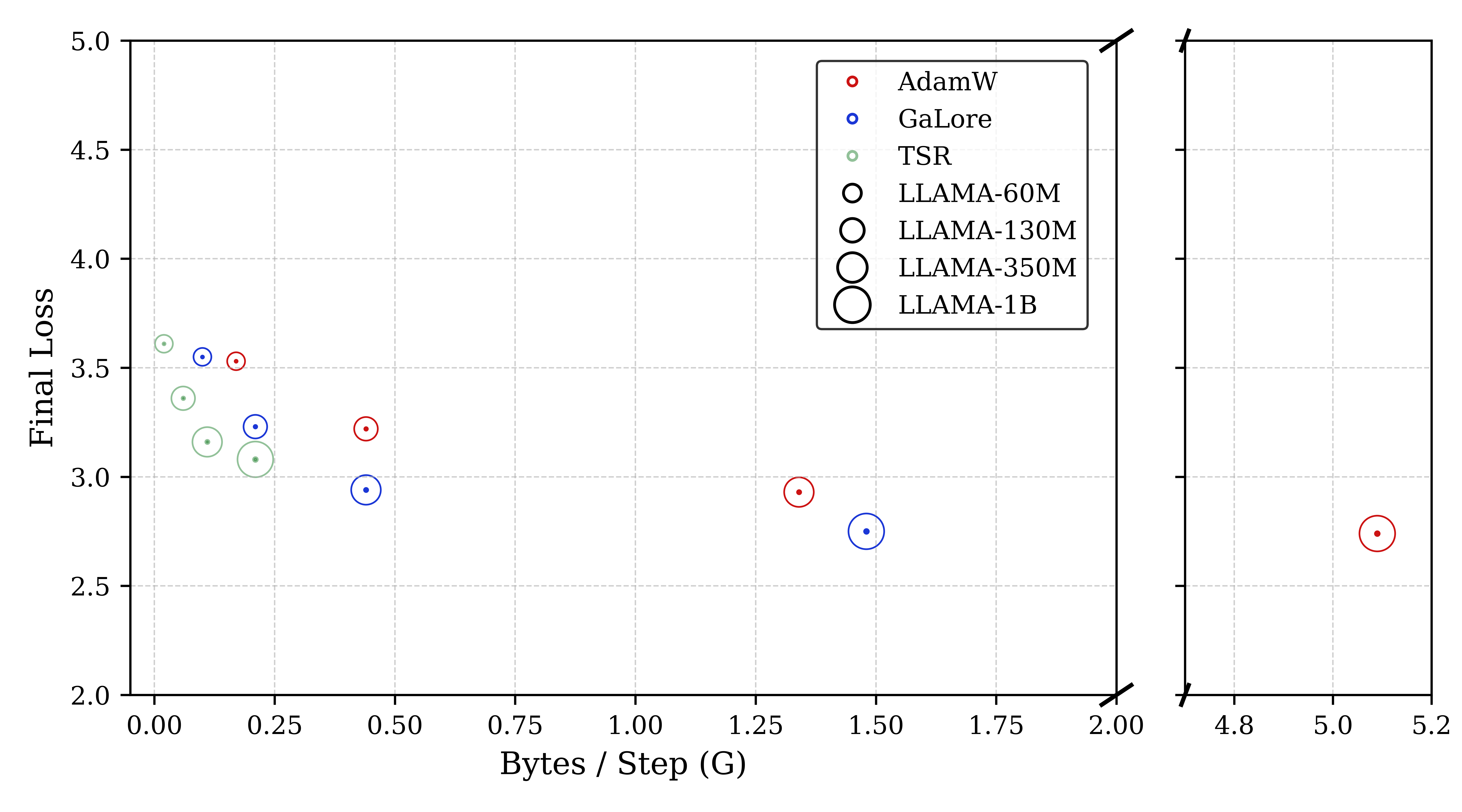}}
    \caption{
\textbf{Loss--communication Pareto frontiers across model scales.}
Final pretraining loss versus communicated Bytes/Step for 60M/130M/350M/1B models.
TSR shifts the frontier toward lower communication for competitive loss, relative to AdamW and GaLore.
}
    \label{fig:pareto}
  \end{center}
  \vskip -0.2in
\end{figure}

\textbf{Embedding gradients are a non-trivial communication source, and compressing them matters.}
Embedding layers are often overlooked in low-rank optimization but can dominate communication in smaller models.
Figure~\ref{fig:embedding}(a) shows that embeddings account for a significant fraction of gradient traffic.
Applying TSR to embeddings with embedding-specific $(r_{\text{emb}},K_{\text{emb}})$ further reduces total communication overhead while maintaining competitive pretraining loss (Figure~\ref{fig:embedding}(b)),
supporting a holistic compression strategy that covers both linear and embedding matrices.

\subsection{Ablations: Two-Sidedness, Randomized SVD, and Subspace Refresh Interval}
\label{subsec:abaltion}

We conduct ablation studies on the LLaMA-60M model to investigate the effects of two-sidedness, randomized SVD, and the subspace refresh interval on training convergence. The results are presented in Figure~\ref{fig:ablation}.

\textbf{One-Side vs. Two-Side.} The results indicate that the two-sided approach achieves comparable convergence to the one-sided variant while reducing the total communication volume by approximately two-thirds. This indicates that two-sided compression can substantially decrease communication costs at the expense of only a very small increase in final loss.

\textbf{Normal SVD vs. Randomized SVD.} The results indicate that using randomized SVD has little impact on convergence performance. Randomized SVD achieves a final loss comparable to that of standard SVD while using nearly half of the total communication volume. Moreover, our previous experiments indicate that, for models of the same scale, randomized SVD reduces the time required for subspace refreshes by approximately half.

\textbf{Subspace Refresh Interval.} The results indicate that a subspace refresh interval of 100 steps achieves the best convergence, while intervals of 50 or 200 steps incur degraded performance. Notably, excessively frequent refreshes, such as every 20 steps in our experiments, result in a pronounced increase in communication overhead.

\begin{table*}[th]
  \caption{Fine-tuning results on GLUE benchmark using pre-trained RoBERTa-Base.We report task metrics and bytes/step.}
  \label{tab:finetune}
  \begin{center}
  \setlength{\tabcolsep}{3pt}
    \begin{small}
      \begin{sc}
        \begin{tabular}{lcccccccccc}
          \toprule 
           & Bytes/Step&CoLA&STS-B & MRPC &RTE&SST2 &MNLI& QNLI&QQP&Avg \\
          \midrule
          Adam & 494M&62.24&90.92 & 91.30&79.42&  94.57&87.18 & 92.33& 92.28&86.28\\
          GaLore &158M &59.33&90.82 &90.27&77.98&93.92&87.11&92.42&89.54&85.17\\
          TSR & 20M &61.32&90.72&90.38&77.98&93.92&86.88&92.95&90.16&85.54 \\
          \bottomrule
        \end{tabular}
      \end{sc}
    \end{small}
  \end{center}
  \vskip -0.1in
\end{table*}

\subsection{Fine-tuning Results}
\label{subsec:finetune}

We fine-tuned pre-trained RoBERTa-Base~\cite{liu2019roberta} model on the GLUE tasks, and the results are shown in Table~\ref{tab:finetune}. Our method achieves better or comparable performance to GaLore on 6 of the 8 tasks, while requiring only one-eighth of the bytes per step. Compared to Adam, our method reduces the bytes per step by approximately 25×, while incurring only a 0.86\% drop in average task performance. See the Appendix~\ref{app:exp} for parameter settings and convergence graphs for each dataset.

\begin{figure}[ht]
  \centering
  \begin{minipage}{0.48\columnwidth}
    \centering
    \includegraphics[width=\linewidth]{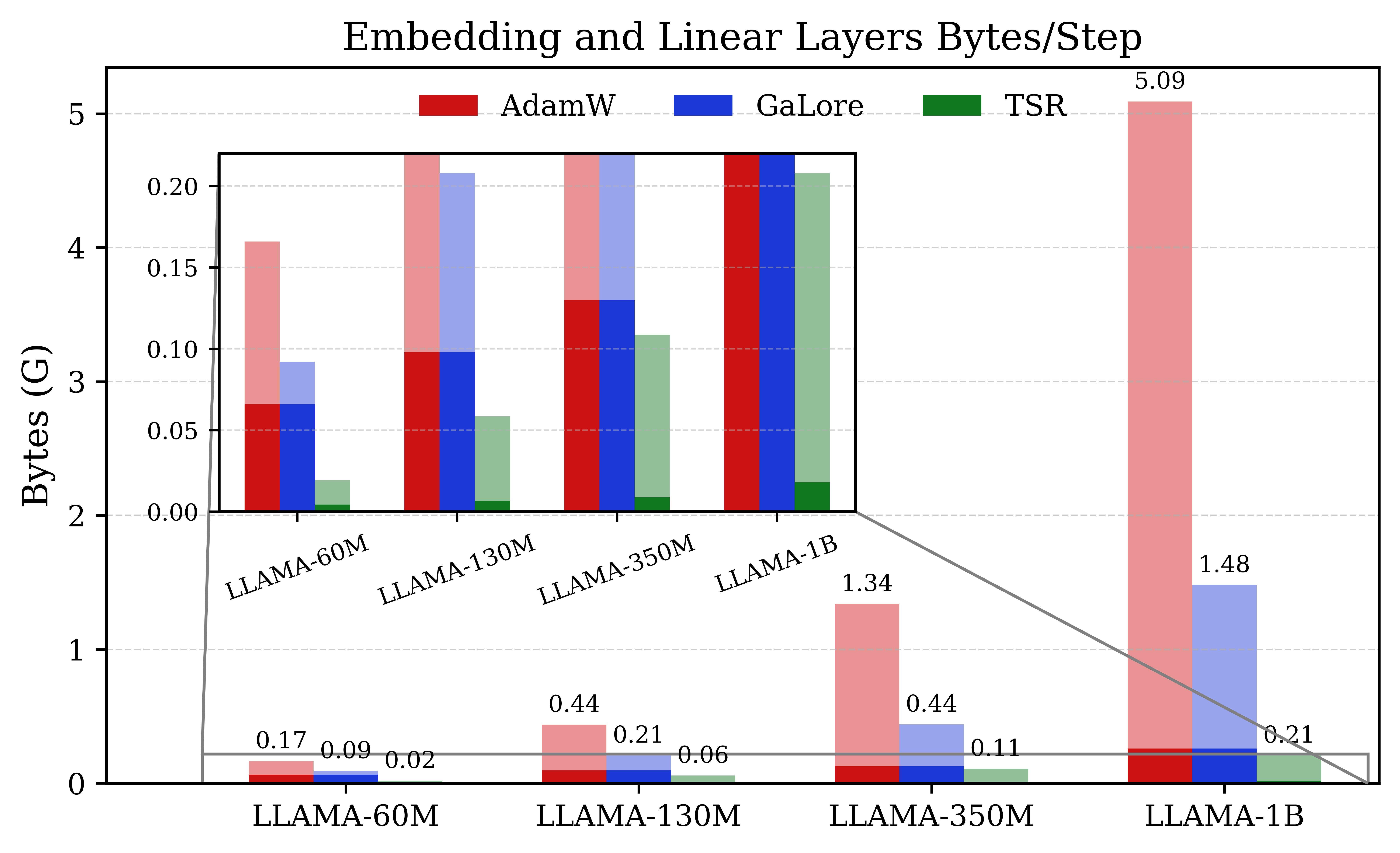} 
    \small (a)
  \end{minipage}
  \hfill
  \begin{minipage}{0.48\columnwidth}
    \centering
    \includegraphics[width=\linewidth]{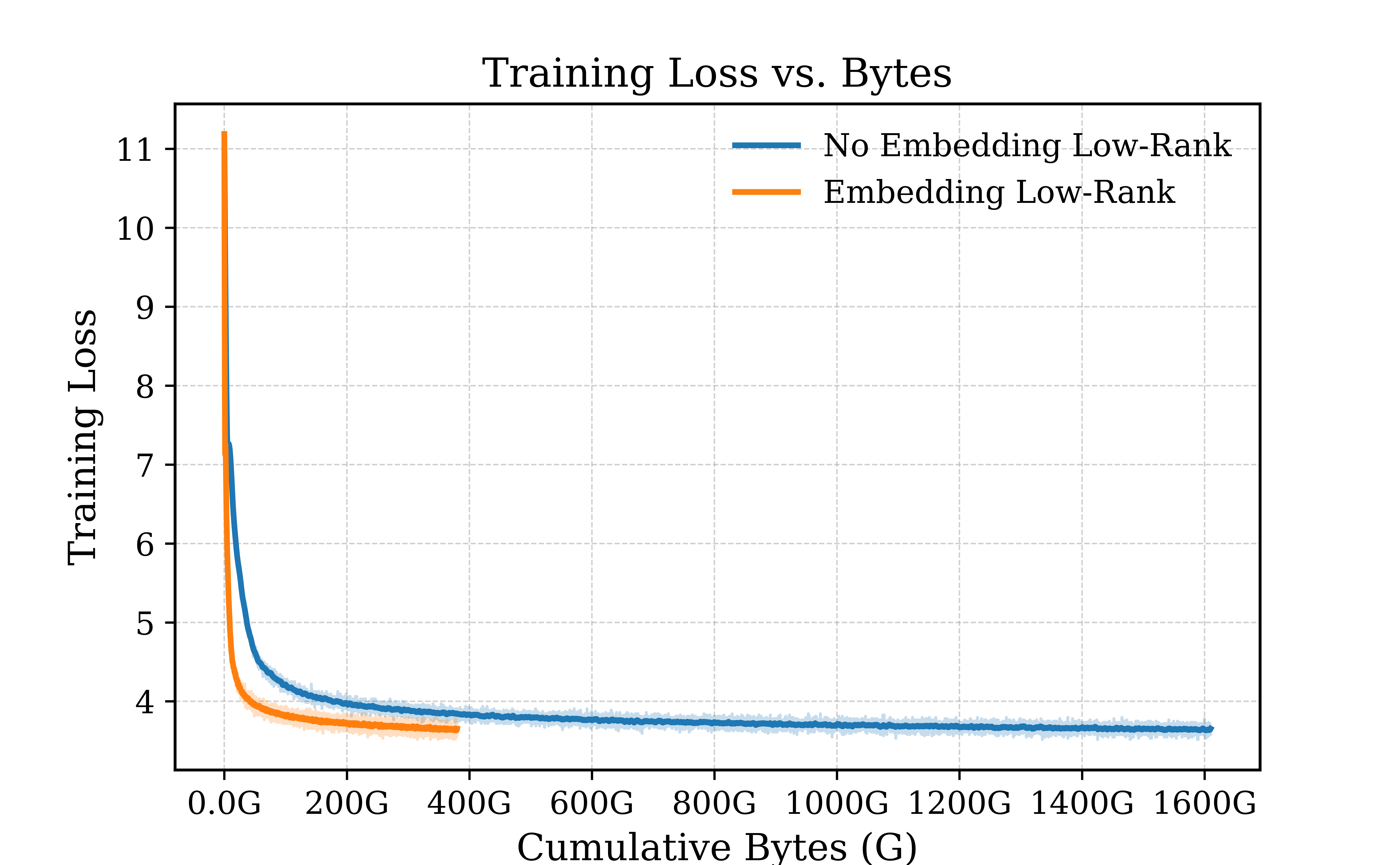} 
    \small (b)
  \end{minipage}

  \caption{
    \textbf{Embedding matters.} 
    (a) Breakdown of bytes per step for the embedding and linear layers across different model sizes. 
    (b) Loss–Bytes curves comparing the use of low-rank compression on the embedding layer.
  }
  \label{fig:embedding}
\end{figure}

\section{Conclusion}
\label{sec:conclusion}

We propose TSR, a two-sided low-rank communication mechanism for distributed training of large language models.
By synchronizing only a compact core $U^\top G V\in\mathbb{R}^{r\times r}$, TSR reduces the dominant per-step communication
from dense $O(mn)$  to $O(r^2)$, while keeping optimizer states in the same low-dimensional core space.
To control refresh-induced peak traffic, TSR adopts a randomized SVD-based refresh that communicates only low-dimensional sketches.
We further extend the same principle to embedding gradients via embedding-specific ranks and refresh schedules.
Empirically, TSR substantially improves bytes-to-loss efficiency in pretraining and fine-tuning.

\bibliography{example_paper}

@inproceedings{tang20211,
  title={1-bit adam: Communication efficient large-scale training with adam’s convergence speed},
  author={Tang, Hanlin and Gan, Shaoduo and Awan, Ammar Ahmad and Rajbhandari, Samyam and Li, Conglong and Lian, Xiangru and Liu, Ji and Zhang, Ce and He, Yuxiong},
  booktitle={International Conference on Machine Learning},
  pages={10118--10129},
  year={2021},
  organization={PMLR}
}

@inproceedings{narayanan2021efficient,
  title={Efficient large-scale language model training on gpu clusters using megatron-lm},
  author={Narayanan, Deepak and Shoeybi, Mohammad and Casper, Jared and LeGresley, Patrick and Patwary, Mostofa and Korthikanti, Vijay and Vainbrand, Dmitri and Kashinkunti, Prethvi and Bernauer, Julie and Catanzaro, Bryan and others},
  booktitle={Proceedings of the international conference for high performance computing, networking, storage and analysis},
  pages={1--15},
  year={2021}
}

@inproceedings{rajbhandari2020zero,
  title={Zero: Memory optimizations toward training trillion parameter models},
  author={Rajbhandari, Samyam and Rasley, Jeff and Ruwase, Olatunji and He, Yuxiong},
  booktitle={SC20: International Conference for High Performance Computing, Networking, Storage and Analysis},
  pages={1--16},
  year={2020},
  organization={IEEE}
}

@article{alistarh2017qsgd,
  title={QSGD: Communication-efficient SGD via gradient quantization and encoding},
  author={Alistarh, Dan and Grubic, Demjan and Li, Jerry and Tomioka, Ryota and Vojnovic, Milan},
  journal={Advances in neural information processing systems},
  volume={30},
  year={2017}
}

@article{dettmers20218,
  title={8-bit optimizers via block-wise quantization},
  author={Dettmers, Tim and Lewis, Mike and Shleifer, Sam and Zettlemoyer, Luke},
  journal={arXiv preprint arXiv:2110.02861},
  year={2021}
}

@article{lin2017deep,
  title={Deep gradient compression: Reducing the communication bandwidth for distributed training},
  author={Lin, Yujun and Han, Song and Mao, Huizi and Wang, Yu and Dally, William J},
  journal={arXiv preprint arXiv:1712.01887},
  year={2017}
}

@article{vogels2019powersgd,
  title={PowerSGD: Practical low-rank gradient compression for distributed optimization},
  author={Vogels, Thijs and Karimireddy, Sai Praneeth and Jaggi, Martin},
  journal={Advances in Neural Information Processing Systems},
  volume={32},
  year={2019}
}

@inproceedings{mu2022hylo,
  title={HyLo: a hybrid low-rank natural gradient descent method},
  author={Mu, Baorun and Soori, Saeed and Can, Bugra and G{\"u}rb{\"u}zbalaban, Mert and Dehnavi, Maryam Mehri},
  booktitle={SC22: International Conference for High Performance Computing, Networking, Storage and Analysis},
  pages={1--16},
  year={2022},
  organization={IEEE}
}

@article{zhao2024galore,
  title={Galore: Memory-efficient llm training by gradient low-rank projection},
  author={Zhao, Jiawei and Zhang, Zhenyu and Chen, Beidi and Wang, Zhangyang and Anandkumar, Anima and Tian, Yuandong},
  journal={arXiv preprint arXiv:2403.03507},
  year={2024}
}

@article{su2025galore,
  title={Galore 2: Large-scale llm pre-training by gradient low-rank projection},
  author={Su, DiJia and Gu, Andrew and Xu, Jane and Tian, Yuandong and Zhao, Jiawei},
  journal={arXiv preprint arXiv:2504.20437},
  year={2025}
}

@inproceedings{jaiswallow,
  title={From Low Rank Gradient Subspace Stabilization to Low-Rank Weights: Observations, Theories, and Applications},
  author={JAISWAL, AJAY KUMAR and Wang, Yifan and Yin, Lu and Liu, Shiwei and Chen, Runjin and Zhao, Jiawei and Grama, Ananth and Tian, Yuandong and Wang, Zhangyang},
  booktitle={Forty-second International Conference on Machine Learning},
  year={2025}
}

@inproceedings{zhang2025breaking,
  title={Breaking the frozen subspace: Importance sampling for low-rank optimization in llm pretraining},
  author={Zhang, Haochen and Yin, Junze and Wang, Guanchu and Liu, Zirui and Yang, Lin and Zhang, Tianyi and Shrivastava, Anshumali and Braverman, Vladimir},
  booktitle={The Thirty-ninth Annual Conference on Neural Information Processing Systems},
  year={2025}
}

@article{he2024subspace,
  title={Subspace optimization for large language models with convergence guarantees},
  author={He, Yutong and Li, Pengrui and Hu, Yipeng and Chen, Chuyan and Yuan, Kun},
  journal={arXiv preprint arXiv:2410.11289},
  year={2024}
}

@article{das2024natural,
  title={Natural galore: Accelerating galore for memory-efficient llm training and fine-tuning},
  author={Das, Arijit},
  journal={arXiv preprint arXiv:2410.16029},
  year={2024}
}

@inproceedings{mo2025parameter,
  title={Parameter and memory efficient pretraining via low-rank riemannian optimization},
  author={Mo, Zhanfeng and Huang, Long-Kai and Pan, Sinno Jialin},
  booktitle={The Thirteenth International Conference on Learning Representations},
  year={2025}
}

@article{hu2022lora,
  title={Lora: Low-rank adaptation of large language models.},
  author={Hu, Edward J and Shen, Yelong and Wallis, Phillip and Allen-Zhu, Zeyuan and Li, Yuanzhi and Wang, Shean and Wang, Lu and Chen, Weizhu and others},
  journal={ICLR},
  volume={1},
  number={2},
  pages={3},
  year={2022}
}

@article{shoeybi2019megatron,
  title={Megatron-lm: Training multi-billion parameter language models using model parallelism},
  author={Shoeybi, Mohammad and Patwary, Mostofa and Puri, Raul and LeGresley, Patrick and Casper, Jared and Catanzaro, Bryan},
  journal={arXiv preprint arXiv:1909.08053},
  year={2019}
}

@inproceedings{jiang2024megascale,
  title={$\{$MegaScale$\}$: Scaling large language model training to more than 10,000 $\{$GPUs$\}$},
  author={Jiang, Ziheng and Lin, Haibin and Zhong, Yinmin and Huang, Qi and Chen, Yangrui and Zhang, Zhi and Peng, Yanghua and Li, Xiang and Xie, Cong and Nong, Shibiao and others},
  booktitle={21st USENIX Symposium on Networked Systems Design and Implementation (NSDI 24)},
  pages={745--760},
  year={2024}
}

@inproceedings{rajbhandari2021zero,
  title={Zero-infinity: Breaking the gpu memory wall for extreme scale deep learning},
  author={Rajbhandari, Samyam and Ruwase, Olatunji and Rasley, Jeff and Smith, Shaden and He, Yuxiong},
  booktitle={Proceedings of the international conference for high performance computing, networking, storage and analysis},
  pages={1--14},
  year={2021}
}

@article{zhao2023pytorch,
  title={Pytorch fsdp: experiences on scaling fully sharded data parallel},
  author={Zhao, Yanli and Gu, Andrew and Varma, Rohan and Luo, Liang and Huang, Chien-Chin and Xu, Min and Wright, Less and Shojanazeri, Hamid and Ott, Myle and Shleifer, Sam and others},
  journal={arXiv preprint arXiv:2304.11277},
  year={2023}
}

@inproceedings{karimireddy2019error,
  title={Error feedback fixes signsgd and other gradient compression schemes},
  author={Karimireddy, Sai Praneeth and Rebjock, Quentin and Stich, Sebastian and Jaggi, Martin},
  booktitle={International conference on machine learning},
  pages={3252--3261},
  year={2019},
  organization={PMLR}
}

@article{zhang2023adalora,
  title={Adalora: Adaptive budget allocation for parameter-efficient fine-tuning},
  author={Zhang, Qingru and Chen, Minshuo and Bukharin, Alexander and Karampatziakis, Nikos and He, Pengcheng and Cheng, Yu and Chen, Weizhu and Zhao, Tuo},
  journal={arXiv preprint arXiv:2303.10512},
  year={2023}
}

@article{dettmers2023qlora,
  title={Qlora: Efficient finetuning of quantized llms},
  author={Dettmers, Tim and Pagnoni, Artidoro and Holtzman, Ari and Zettlemoyer, Luke},
  journal={Advances in neural information processing systems},
  volume={36},
  pages={10088--10115},
  year={2023}
}

@inproceedings{liu2024dora,
  title={Dora: Weight-decomposed low-rank adaptation},
  author={Liu, Shih-Yang and Wang, Chien-Yi and Yin, Hongxu and Molchanov, Pavlo and Wang, Yu-Chiang Frank and Cheng, Kwang-Ting and Chen, Min-Hung},
  booktitle={Forty-first International Conference on Machine Learning},
  year={2024}
}

@inproceedings{george2024tensor,
  title={Tensor-GaLore: Memory-Efficient Training via Gradient Tensor Decomposition},
  author={George, Robert Joseph and Pitt, David and Zhao, Jiawei and Kossaifi, Jean and Luo, Cheng and Tian, Yuandong and Anandkumar, Anima},
  booktitle={OPT 2024: Optimization for Machine Learning},
  year={2024}
}

@article{liang2024memory,
  title={Memory-efficient llm training with online subspace descent},
  author={Liang, Kaizhao and Liu, Bo and Chen, Lizhang and Liu, Qiang},
  journal={Advances in Neural Information Processing Systems},
  volume={37},
  pages={64412--64432},
  year={2024}
}

@inproceedings{rasley2020deepspeed,
  title={Deepspeed: System optimizations enable training deep learning models with over 100 billion parameters},
  author={Rasley, Jeff and Rajbhandari, Samyam and Ruwase, Olatunji and He, Yuxiong},
  booktitle={Proceedings of the 26th ACM SIGKDD international conference on knowledge discovery \& data mining},
  pages={3505--3506},
  year={2020}
}

@article{chen2025greedy,
  title={Greedy low-rank gradient compression for distributed learning with convergence guarantees},
  author={Chen, Chuyan and He, Yutong and Li, Pengrui and Jia, Weichen and Yuan, Kun},
  journal={arXiv preprint arXiv:2507.08784},
  year={2025}
}

@article{brown2020language,
  title={Language models are few-shot learners},
  author={Brown, Tom and Mann, Benjamin and Ryder, Nick and Subbiah, Melanie and Kaplan, Jared D and Dhariwal, Prafulla and Neelakantan, Arvind and Shyam, Pranav and Sastry, Girish and Askell, Amanda and others},
  journal={Advances in neural information processing systems},
  volume={33},
  pages={1877--1901},
  year={2020}
}

@article{bommasani2021opportunities,
  title={On the opportunities and risks of foundation models},
  author={Bommasani, Rishi},
  journal={arXiv preprint arXiv:2108.07258},
  year={2021}
}

@article{touvron2023llama,
  title={Llama 2: Open foundation and fine-tuned chat models},
  author={Touvron, Hugo and Martin, Louis and Stone, Kevin and Albert, Peter and Almahairi, Amjad and Babaei, Yasmine and Bashlykov, Nikolay and Batra, Soumya and Bhargava, Prajjwal and Bhosale, Shruti and others},
  journal={arXiv preprint arXiv:2307.09288},
  year={2023}
}

@article{raffel2020exploring,
  title={Exploring the limits of transfer learning with a unified text-to-text transformer},
  author={Raffel, Colin and Shazeer, Noam and Roberts, Adam and Lee, Katherine and Narang, Sharan and Matena, Michael and Zhou, Yanqi and Li, Wei and Liu, Peter J},
  journal={Journal of machine learning research},
  volume={21},
  number={140},
  pages={1--67},
  year={2020}
}

@inproceedings{wang2018glue,
  title={GLUE: A multi-task benchmark and analysis platform for natural language understanding},
  author={Wang, Alex and Singh, Amanpreet and Michael, Julian and Hill, Felix and Levy, Omer and Bowman, Samuel},
  booktitle={Proceedings of the 2018 EMNLP workshop BlackboxNLP: Analyzing and interpreting neural networks for NLP},
  pages={353--355},
  year={2018}
}

@article{liu2019roberta,
  title={Roberta: A robustly optimized bert pretraining approach},
  author={Liu, Yinhan and Ott, Myle and Goyal, Naman and Du, Jingfei and Joshi, Mandar and Chen, Danqi and Levy, Omer and Lewis, Mike and Zettlemoyer, Luke and Stoyanov, Veselin},
  journal={arXiv preprint arXiv:1907.11692},
  year={2019}
}
\bibliographystyle{plainnat}

\newpage
\appendix
\paragraph{Road Map.}
Appendix~\ref{app_sec:relate} further elaborates on the relevant work.
Appendix~\ref{app_sec:proofs} provides a detailed proof of the convergence of the algorithm, including the TSR-SGD version of the algorithm, and Appendix~\ref{app:exp} shows more experimental setups and experimental graphs.

\section{Related Works}
\label{app_sec:relate}
\paragraph{Communication Bottlenecks in Large-Scale Pretraining of LLMs}
The training of large language models at scale is often limited not by compute, but by communication overhead. When hundreds or thousands of GPUs are used in distributed data-parallel training, the cost of synchronizing gradients and model updates over interconnects (PCIe, NVLink, InfiniBand/Ethernet) can dominate the iteration time~\cite{tang20211,narayanan2021efficient}. In modern LLM pretraining runs, network bandwidth becomes a primary bottleneck – for example, multi-node training must send enormous gradients (on the order of billions of parameters) across relatively slower inter-node links, leading to idle compute while GPUs wait on data exchange~\cite{narayanan2021efficient}. This has motivated a long line of communication-efficient optimization techniques to alleviate bandwidth limitations in distributed training of foundation models.

\textbf{Optimizer state sharding}. One approach is to eliminate redundant data transfers by partitioning or sharding optimizer states across workers. ZeRO (Zero Redundancy Optimizer) introduced by Rajbhandari et al~\cite{rajbhandari2020zero}. shards the gradient histories and optimizer momentum/variance so that each GPU only stores (and communicates) a fraction of the full state, instead of replicating it on all GPUs. By stripping away duplicate communication and memory, ZeRO achieved near-linear scaling for models up to 100+ billion parameters, retaining low communication volume despite massive model size~\cite{rajbhandari2020zero}. This idea underpins Fully Sharded Data Parallel schemes, ensuring that bandwidth usage grows sub-linearly with the number of accelerators.

\textbf{Gradient quantization and compression}. Another class of methods reduces the number of bits or values that need to be exchanged each step. For instance, QSGD (Alistarh et al., 2017) proposed quantizing gradients to low-bit representations, showing one can trade off communication bandwidth for convergence time by sending fewer bits per gradient component~\cite{alistarh2017qsgd}. Such lossy compression directly targets the “high bandwidth cost of communicating gradient updates” in parallel SGD~\cite{alistarh2017qsgd}. Following this, researchers developed 1-bit SGD/Adam techniques that transmit only a 1-bit sign or compressed form of each gradient with error compensation. Notably, 1-bit Adam~\cite{tang20211}demonstrated up to 5× reduction in communication volume with no loss in convergence speed compared to uncompressed Adam. This was achieved by accumulating quantization errors and leveraging the stability of Adam’s second moment estimates. Similarly, techniques for 8-bit compression of gradients/updates have been explored – for example, Dettmers et al. developed 8-bit optimizers that store and communicate optimizer statistics at 8-bit precision while preserving 32-bit performance~\cite{dettmers20218}. By quantizing gradient exchanges (e.g. 8-bit or even 4-bit), these methods cut bandwidth requirements substantially, at the cost of minimal noise that can be mitigated with proper tuning.

\textbf{Gradient sparsification.} In parallel, researchers have investigated sending only a subset of gradient elements to reduce traffic. Deep Gradient Compression (DGC)~\cite{lin2017deep} showed that 99.9\% of gradient values in distributed training can be deemed redundant and pruned away with little impact on model quality. DGC aggressively sparsifies gradients (e.g. by transmitting only the top-$k$ magnitude components) and uses techniques like momentum correction and warm-up to preserve accuracy~\cite{lin2017deep}. The result was a 270×–600× compression of gradient data (e.g. reducing a ResNet-50 gradient from 97 MB to 0.35 MB) without loss of accuracy~\cite{lin2017deep}. Such extreme compression allowed training on bandwidth-limited 1 Gbps Ethernet clusters that would otherwise be prohibitively slow~\cite{lin2017deep}. Gradient sparsification and pruning, combined with error-feedback mechanisms, continue to be effective for cutting down communication in large-scale setups.

\textbf{Low-rank gradient communication.} More recently, structured approaches have been proposed to compress gradients by exploiting linear correlations. PowerSGD~\cite{vogels2019powersgd} is a representative method that approximates the gradient matrix in each layer by a low-rank factorization and only communicates the low-rank factors. By using a small rank (e.g. $r=1$ or $4$) and a power iteration method to find these factors, PowerSGD can drastically reduce bandwidth while maintaining final accuracy close to full-precision SGD. It was one of the first schemes to demonstrate actual wall-clock speedups in distributed training owing to communication savings~\cite{vogels2019powersgd}. Related work has combined low-rank compression with second-order information – for example, using a natural gradient low-rank scheme that preconditions updates. Mu et al. (2023) showed that a hybrid low-rank natural gradient method can reduce communication time by over an order of magnitude (up to $10.7\times$) on multi-GPU training, thanks to smaller transmitted updates and fewer synchronization steps~\cite{mu2022hylo}. Overall, by compressing gradients via quantization, sparsity, or low-rank approximation, distributed LLM training can significantly mitigate network bottlenecks and push closer to the limits of hardware bandwidth. These communication-efficient optimizers have become crucial as model sizes and cluster scales continue to grow, ensuring that networking infrastructure does not become the rate limiter in large-scale pretraining.
\paragraph{Low-Rank Optimization for Memory-Efficient Training}
Scaling up model size not only stresses communication, but also introduces severe memory and compute challenges. Modern LLMs with billions of parameters demand enormous memory to store model weights, optimizer states, and gradients. For example, training a 7B-parameter transformer from scratch with Adam optimizer can require on the order of 60 GB of memory per GPU, with optimizer momentum and variance states taking much more space than the weights themselves~\cite{zhao2024galore}. To address this, researchers have explored low-rank optimization methods that reduce memory and compute usage by restricting updates to low-dimensional subspaces.

\textbf{Low-Rank Adaptation (LoRA). }A classical approach for memory-efficient training is to constrain weight updates to a low-rank form. LoRA~\cite{hu2022lora} injects trainable low-rank matrices into each layer’s weights instead of updating the full weight matrix~\cite{zhao2024galore}. In LoRA, a pre-trained weight $W_0$ is frozen and two small rank-$r$ matrices $U,V$ are learned such that $W = W_0 + U V^T$. Because $U$ and $V$ have far fewer parameters than $W$ (with $r \ll \text{dim}(W)$), the number of trainable parameters – and crucially the optimizer states for them – are dramatically reduced~\cite{zhao2024galore}. LoRA proved to be an efficient fine-tuning strategy, cutting memory usage while maintaining decent performance on downstream tasks. However, LoRA’s low-rank constraint can become a limitation in full training: the model is confined to a tiny subspace of the full parameter space. Recent studies found that LoRA fine-tuning often cannot match the accuracy of full-model tuning, and in pre-training from scratch it may even require a full-rank warm-up phase before low-rank updates are viable~\cite{zhao2024galore}. In essence, if the true optimal weights are not low-rank or if the low-rank reparameterization alters training dynamics, a fixed small subspace can hurt convergence. This motivates methods that preserve the flexibility of full-rank training while still reaping memory savings from low-rank structures.

\textbf{Gradient low-rank projection (GaLore).} Instead of freezing weights to a low-rank form, one can project gradients onto a low-rank subspace during training. GaLore~\cite{zhao2024galore} is an optimizer that maintains full-rank model weights but restricts each weight’s gradient to a learned low-rank subspace. The key insight is that as training progresses, the gradients of large models often lie in a much lower-dimensional subspace than the parameter count – they have a slowly changing, low-rank structure. GaLore periodically computes a small set of basis vectors (e.g. via SVD on recent gradients) and projects incoming gradients onto this subspace, then applies the optimizer (Adam, Adagrad, etc.) in the compressed space~\cite{zhao2024galore}. By doing so, the memory-heavy optimizer states (momentum, etc.) are only kept for the low-rank coordinates rather than for every parameter. This approach allows full-parameter training (no accuracy drop from restricting model capacity) but yields major memory savings: GaLore reports up to a 65\% reduction in optimizer memory compared to standard Adam, and up to 82\% reduction when combined with 8-bit quantized states~\cite{zhao2024galore}. In practice, GaLore enabled, for the first time, pre-training a 7B LLM on a single 24 GB GPU without any offloading or model-parallel tricks~\cite{zhao2024galore}. Compared to LoRA, which limits the model to a fixed low-rank, GaLore’s online subspace descent allows the subspace to evolve – e.g. updating the projection basis every few hundred steps – which was shown to preserve model quality while still reducing memory by \~30\% relative to LoRA’s approach.

\textbf{Extensions of GaLore.} The idea of low-rank optimizer updates has spurred many follow-ups. Some works add quantization and other structure on top of GaLore – for example, 8-bit GaLore quantizes the projected gradients to further cut memory, and researchers have explored using higher-order tensor decompositions (Tensor-GaLore~\cite{george2024tensor}) to compress gradients along multiple modes~\cite{su2025galore}. These improvements address practical bottlenecks: GaLore’s original implementation incurred non-negligible SVD computation for subspace updates and didn’t natively support sharded data parallel. In GaLore 2~\cite{su2025galore}, the framework was optimized with fast randomized SVD updates and integration into PyTorch’s Fully Sharded Data Parallel, making it efficient and scalable for real-world runs. GaLore 2 demonstrated the ability to pre-train a LLaMA-7B model on 500 billion tokens using low-rank gradient projection, indicating that such methods can indeed support large-scale training without hitting memory limits. Beyond GaLore itself, researchers have looked at the connection between gradient subspaces and weight matrices.~\cite{jaiswallow} observed that as training converges, certain layers’ gradients stabilize in a low-rank subspace and correspondingly the weights exhibit implicit low-rank structure. They introduced WeLore (Weight Low-Rank Projection), a one-shot compression technique that leverages these observations: after pretraining, it categorizes weight matrices into those that can be compressed (low-rank) and those that cannot, and compresses the former for memory-efficient fine-tuning or inference~\cite{jaiswallow}. WeLore essentially unifies model compression with the low-rank training insights from GaLore, achieving significant memory reduction with minimal performance drop by tailoring the rank per layer.

\textbf{Dynamic subspace and theoretical advances.} An important question for low-rank training is how to choose and update the subspace. If the subspace is too restrictive or is not refreshed, training can stagnate.~\cite{zhang2025breaking} highlighted this in “Breaking the Frozen Subspace”: they found that methods like GaLore which pick the dominant gradient subspace can suffer because the top principal components of the gradient stop changing over time, effectively freezing the update directions. To counter this, they proposed an importance sampling strategy that occasionally forces exploration of new random directions outside the current subspace, and proved that this yields convergence guarantees that the purely dominant subspace approach lacks. Empirically, their method outperformed prior low-rank optimizers in LLM pretraining, underlining the need to keep the optimization subspace from becoming too static. On the theoretical side,~\cite{he2024subspace} examined GaLore’s convergence properties and showed that GaLore may fail to converge in stochastic training unless certain conditions (large batch sizes or isotropic gradient noise) hold. They introduced a variant called GoLore (Gradient random Low-Rank projection) that uses randomization in the projection step, and proved it converges reliably in standard SGD settings. This work provides formal guarantees for subspace optimization algorithms and guides how to design low-rank updates that are stable in theory as well as practice.

Finally, researchers have begun integrating low-rank methods with second-order optimization and manifold geometry. Natural GaLore~\cite{das2024natural} is one such approach that applies a Riemannian preconditioning to GaLore’s low-rank updates. By multiplying the low-rank gradient by the inverse Fisher information matrix (estimated efficiently via the Woodbury identity), Natural GaLore accelerates convergence without additional memory overhead. This can be seen as performing natural gradient descent within the low-rank subspace, yielding faster loss reduction especially when the training budget (number of steps) is limited. In a similar vein,~\cite{mo2025parameter}proposed LORO (Low-Rank Riemannian Optimizer) for full low-rank pretraining. Instead of maintaining a full-rank model, LORO parameterizes each weight matrix as the product of two rank-$r$ factors from scratch and optimizes these factors jointly on the appropriate Riemannian manifold. By updating the factor pair in tandem along the manifold’s steepest descent direction (avoiding any reconstruction of full-sized gradients), LORO manages to learn low-rank models that match or even exceed the quality of full models, while using far less memory. Notably, an LLM pre-trained with LORO (rank 256 on a 1B parameter model) reached a slightly better perplexity than a full model, despite using ~54\% less memory and achieving 1.8× faster training and 2.2× faster inference throughput. Such results are encouraging as they suggest that, with proper optimization on the low-rank manifold, we can attain the dual goals of memory efficiency and high performance. In summary, a rich landscape of low-rank optimization methods – from practical GaLore-based techniques to theoretical and second-order enhancements – has emerged to enable memory-efficient training of large models. These works collectively pave the way for optimizers that address both communication and memory bottlenecks in large-scale LLM pretraining, which is precisely the gap our proposed method aims to fill.

\clearpage


\newpage
\section{Detailed Convergence Analysis}
\label{app_sec:proofs}

\subsection{
\texorpdfstring
  {Part 1: One-step tracking recursion (no refresh, before shifting $w_t \to w_{t-1}$)}
  {Part 1: One-step tracking recursion (no refresh, before shifting wt to w(t-1))}
}
\label{subsec:proof_part1}

We analyze a fixed step $t$ under the no-refresh condition, i.e., $U_t=U_{t-1}$ and $V_t=V_{t-1}$.
Recall the core-space momentum update
\begin{equation}
m_t = \beta m_{t-1} + (1-\beta)\bar{C}_t.
\label{eq:mt_update_part1}
\end{equation}

\paragraph{Projected gradient and approximation error.}
Define the projected full gradient
\begin{equation}
P_t \;:=\; U_t U_t^\top \nabla f(w_t)\, V_t V_t^\top .
\label{eq:Pt_def}
\end{equation}
Define the low-rank projection approximation error
\begin{equation}
\Delta_t \;:=\; \mathbb{E}\left\| P_t - \nabla f(w_t) \right\|^2.
\label{eq:Delta_def_part1}
\end{equation}

\paragraph{Assumption (unbiasedness and variance in the projected space).}
Let $\mathcal{F}_{t-1}$ denote the sigma-field generated by the history up to step $t-1$.
Under no refresh, $m_{t-1},U_t,V_t,w_t$ are $\mathcal{F}_{t-1}$-measurable.
We assume the synchronized core induces an unbiased estimate of the projected gradient:
\begin{equation}
\mathbb{E}\!\left[ U_t \bar{C}_t V_t^\top \,\middle|\, \mathcal{F}_{t-1}\right] \;=\; P_t,
\label{eq:unbiased_projected}
\end{equation}
and satisfies the variance bound
\begin{equation}
\mathbb{E}\left\| U_t \bar{C}_t V_t^\top - P_t \right\|^2 \;\le\; \sigma_t^2.
\label{eq:variance_bound_part1}
\end{equation}

\paragraph{Step 1: Expand the squared tracking error.}
Define
\begin{equation}
A_t := U_t m_{t-1}V_t^\top - \nabla f(w_t),\qquad
B_t := U_t \bar{C}_tV_t^\top - \nabla f(w_t).
\label{eq:AtBt_def}
\end{equation}
Using \eqref{eq:mt_update_part1}, we have
\begin{align}
\mathbb{E}\left\|U_t m_t V_t^\top - \nabla f(w_t)\right\|^2
&=\mathbb{E}\left\|\beta A_t + (1-\beta)B_t\right\|^2 \nonumber\\
&= \beta^2 \mathbb{E}\|A_t\|^2 + (1-\beta)^2 \mathbb{E}\|B_t\|^2
+ 2\beta(1-\beta)\mathbb{E}\langle A_t,B_t\rangle .
\label{eq:expand_part1}
\end{align}

\paragraph{Step 2: Handle the cross term $\mathbb{E}\langle A_t,B_t\rangle$.}
Decompose $B_t$ by adding and subtracting $P_t$:
\begin{equation}
B_t
=
\underbrace{\big(U_t \bar{C}_tV_t^\top - P_t\big)}_{=:S_t}
+
\underbrace{\big(P_t - \nabla f(w_t)\big)}_{=:Z_t}.
\label{eq:Bt_decomp}
\end{equation}
Then
\begin{equation}
\mathbb{E}\langle A_t,B_t\rangle
= \mathbb{E}\langle A_t,S_t\rangle + \mathbb{E}\langle A_t,Z_t\rangle.
\label{eq:cross_decomp}
\end{equation}
For the first term, by \eqref{eq:unbiased_projected} and the tower property,
\begin{align}
\mathbb{E}\langle A_t,S_t\rangle
&= \mathbb{E}\Big[\mathbb{E}\big[\langle A_t,\,U_t \bar{C}_tV_t^\top - P_t\rangle \mid \mathcal{F}_{t-1}\big]\Big] \nonumber\\
&= \mathbb{E}\Big[\big\langle A_t,\,\mathbb{E}[U_t \bar{C}_tV_t^\top - P_t\mid \mathcal{F}_{t-1}]\big\rangle\Big] \nonumber\\
&= 0.
\label{eq:cross_noise_zero}
\end{align}
Hence,
\begin{equation}
\mathbb{E}\langle A_t,B_t\rangle = \mathbb{E}\langle A_t,Z_t\rangle.
\label{eq:cross_reduced}
\end{equation}

Next we bound $\mathbb{E}\langle A_t,Z_t\rangle$ via Young's inequality: for any $\alpha>0$,
\begin{equation}
2\langle x,y\rangle \le \alpha\|x\|^2 + \frac{1}{\alpha}\|y\|^2.
\label{eq:young_part1}
\end{equation}
Applying \eqref{eq:young_part1} with $x=A_t$, $y=Z_t$ and multiplying by $\beta(1-\beta)$ gives
\begin{equation}
2\beta(1-\beta)\langle A_t,Z_t\rangle
\le \beta(1-\beta)\alpha\|A_t\|^2 + \beta(1-\beta)\frac{1}{\alpha}\|Z_t\|^2.
\label{eq:cross_young}
\end{equation}
Choose
\begin{equation}
\alpha := \frac{1+\beta}{2\beta},
\label{eq:alpha_choice}
\end{equation}
so that $\beta(1-\beta)\alpha = \frac{1-\beta^2}{2}$ and $\beta(1-\beta)\frac{1}{\alpha} = \frac{2\beta^2(1-\beta)}{1+\beta}$.
Taking expectation and using \eqref{eq:cross_reduced} yields
\begin{equation}
2\beta(1-\beta)\mathbb{E}\langle A_t,B_t\rangle
\le \frac{1-\beta^2}{2}\mathbb{E}\|A_t\|^2
+ \frac{2\beta^2(1-\beta)}{1+\beta}\|Z_t\|^2.
\label{eq:cross_final_part1}
\end{equation}

\paragraph{Step 3: Bound $\mathbb{E}\|B_t\|^2$ by variance and approximation.}
Using \eqref{eq:Bt_decomp} and $\|x+y\|^2\le 2\|x\|^2+2\|y\|^2$,
\begin{align}
\mathbb{E}\|B_t\|^2
&=\mathbb{E}\|S_t + Z_t\|^2 \nonumber\\
&\le 2\mathbb{E}\|S_t\|^2 + 2\|Z_t\|^2 \nonumber\\
&\le 2\sigma_t^2 + 2\Delta_t,
\label{eq:Bt_bound_part1}
\end{align}
where the last inequality uses \eqref{eq:variance_bound_part1} and \eqref{eq:Delta_def_part1}.

\paragraph{Step 4: Combine the bounds.}
Plugging \eqref{eq:cross_final_part1} and \eqref{eq:Bt_bound_part1} into \eqref{eq:expand_part1}, we obtain
\begin{align}
\mathbb{E}\|U_t m_t V_t^\top - \nabla f(w_t)\|^2
&\le \beta^2 \mathbb{E}\|A_t\|^2
+ (1-\beta)^2\big(2\sigma_t^2 + 2\Delta_t\big)
+ \frac{1-\beta^2}{2}\mathbb{E}\|A_t\|^2
+ \frac{2\beta^2(1-\beta)}{1+\beta}\Delta_t \nonumber\\
&= \frac{1+\beta^2}{2}\,\mathbb{E}\|A_t\|^2
+ 2(1-\beta)^2 \sigma_t^2
+ \left(2(1-\beta)^2 + \frac{2\beta^2(1-\beta)}{1+\beta}\right)\Delta_t \nonumber\\
&= \frac{1+\beta^2}{2}\,\mathbb{E}\|U_t m_{t-1}V_t^\top - \nabla f(w_t)\|^2
+ 2(1-\beta)^2 \sigma_t^2
+ \frac{2(1-\beta)}{1+\beta}\,\Delta_t.
\label{eq:part1_final_recursion}
\end{align}

\subsection{
\texorpdfstring
  {Part 2: Shifting $\nabla f(w_t)$ to $\nabla f(w_{t-1})$ (no refresh)}
  {Part 2: Shifting grad f(wt) to grad f(w(t-1)) (no refresh)}
}

\label{subsec:proof_part2}

We continue under the no-refresh condition ($U_t=U_{t-1}$ and $V_t=V_{t-1}$).
Starting from \eqref{eq:part1_final_recursion}, the remaining task is to upper bound
$\mathbb{E}\|U_t m_{t-1}V_t^\top-\nabla f(w_t)\|^2$ by a term evaluated at $w_{t-1}$ plus a smoothness error.

\paragraph{Step 1: A quadratic splitting inequality.}
For any $\gamma>0$ and any matrices $X,Y$ of the same shape,
\begin{align}
\|X+Y\|^2
&= \|X\|^2 + \|Y\|^2 + 2\langle X,Y\rangle \nonumber\\
&\le \|X\|^2 + \|Y\|^2 + \gamma\|X\|^2 + \frac{1}{\gamma}\|Y\|^2 
\qquad \text{(by Young: $2\langle X,Y\rangle \le \gamma\|X\|^2 + \frac{1}{\gamma}\|Y\|^2$)} \nonumber\\
&= (1+\gamma)\|X\|^2 + \left(1+\frac{1}{\gamma}\right)\|Y\|^2 .
\label{eq:split_xy_part2}
\end{align}

\paragraph{Step 2: Apply the split with $X=U_t m_{t-1}V_t^\top-\nabla f(w_{t-1})$.}
Note that
\begin{align}
U_t m_{t-1}V_t^\top - \nabla f(w_t)
&=
\Big(U_t m_{t-1}V_t^\top - \nabla f(w_{t-1})\Big)
+
\Big(\nabla f(w_{t-1})-\nabla f(w_t)\Big).
\label{eq:decomp_shift_part2}
\end{align}
Applying \eqref{eq:split_xy_part2} to \eqref{eq:decomp_shift_part2} gives, for any $\gamma>0$,
\begin{align}
\left\|U_t m_{t-1}V_t^\top - \nabla f(w_t)\right\|^2
&\le
(1+\gamma)\left\|U_t m_{t-1}V_t^\top - \nabla f(w_{t-1})\right\|^2 \nonumber\\
&\quad + \left(1+\frac{1}{\gamma}\right)\left\|\nabla f(w_t)-\nabla f(w_{t-1})\right\|^2.
\label{eq:shift_bound_part2_gamma}
\end{align}
Taking expectation on both sides preserves the inequality:
\begin{align}
\mathbb{E}\left\|U_t m_{t-1}V_t^\top - \nabla f(w_t)\right\|^2
&\le
(1+\gamma)\,\mathbb{E}\left\|U_t m_{t-1}V_t^\top - \nabla f(w_{t-1})\right\|^2 \nonumber\\
&\quad + \left(1+\frac{1}{\gamma}\right)\mathbb{E}\left\|\nabla f(w_t)-\nabla f(w_{t-1})\right\|^2.
\label{eq:shift_bound_part2_expect}
\end{align}

\paragraph{Step 3: Use $L$-smoothness to bound the gradient difference.}
Under $L$-smoothness (Assumption 2 in the main paper),
\begin{equation}
\left\|\nabla f(w_t)-\nabla f(w_{t-1})\right\|
\le L\left\|w_t-w_{t-1}\right\|.
\label{eq:smooth_grad_diff}
\end{equation}
Squaring both sides yields
\begin{equation}
\left\|\nabla f(w_t)-\nabla f(w_{t-1})\right\|^2
\le L^2\left\|w_t-w_{t-1}\right\|^2.
\label{eq:smooth_grad_diff_sq}
\end{equation}
Substituting \eqref{eq:smooth_grad_diff_sq} into \eqref{eq:shift_bound_part2_expect} gives
\begin{align}
\mathbb{E}\left\|U_t m_{t-1}V_t^\top - \nabla f(w_t)\right\|^2
&\le
(1+\gamma)\,\mathbb{E}\left\|U_t m_{t-1}V_t^\top - \nabla f(w_{t-1})\right\|^2 \nonumber\\
&\quad + \left(1+\frac{1}{\gamma}\right)L^2\,\mathbb{E}\left\|w_t-w_{t-1}\right\|^2.
\label{eq:shift_bound_part2_smooth}
\end{align}

\paragraph{Step 4: Choose $\gamma=\frac{1-\beta^2}{4}$ and plug into Part 1.}
Set
\begin{equation}
\gamma := \frac{1-\beta^2}{4},
\label{eq:gamma_choice_part2}
\end{equation}
so that
\begin{equation}
1+\gamma = 1+\frac{1-\beta^2}{4},\qquad
1+\frac{1}{\gamma} = 1+\frac{4}{1-\beta^2}=\frac{5-\beta^2}{1-\beta^2}.
\label{eq:gamma_factors_part2}
\end{equation}
Combining \eqref{eq:part1_final_recursion} with \eqref{eq:shift_bound_part2_smooth} and \eqref{eq:gamma_factors_part2} yields
\begin{align}
\mathbb{E}\|U_t m_t V_t^\top - \nabla f(w_t)\|^2
&\le \frac{1+\beta^2}{2}\left(1+\frac{1-\beta^2}{4}\right)
\mathbb{E}\left\|U_t m_{t-1}V_t^\top - \nabla f(w_{t-1})\right\|^2 \nonumber\\
&\quad + \frac{1+\beta^2}{2}\cdot\frac{5-\beta^2}{1-\beta^2}\,L^2\,\mathbb{E}\|w_t-w_{t-1}\|^2 \nonumber\\
&\quad + 2(1-\beta)^2\sigma_t^2 + \frac{2(1-\beta)}{1+\beta}\Delta_t.
\label{eq:part2_plug_in}
\end{align}

\paragraph{Step 5: Simplify the leading coefficients (no skipped algebra).}
For the contraction coefficient in the first line of \eqref{eq:part2_plug_in},
\begin{align}
\frac{1+\beta^2}{2}\left(1+\frac{1-\beta^2}{4}\right)
&=
\frac{1+\beta^2}{2}\cdot\frac{5-\beta^2}{4}
=
\frac{(1+\beta^2)(5-\beta^2)}{8}
=
\frac{5+4\beta^2-\beta^4}{8}.
\label{eq:coeff_expand_part2}
\end{align}
Moreover,
\begin{align}
\frac{3+\beta^2}{4} - \frac{5+4\beta^2-\beta^4}{8}
&=
\frac{2(3+\beta^2) - (5+4\beta^2-\beta^4)}{8} \nonumber\\
&=
\frac{6+2\beta^2-5-4\beta^2+\beta^4}{8}
=
\frac{(1-\beta^2)^2}{8}
\;\ge\;0,
\label{eq:coeff_compare_part2}
\end{align}
which implies
\begin{equation}
\frac{5+4\beta^2-\beta^4}{8}
\le
\frac{3+\beta^2}{4}
=
1-\frac{1-\beta^2}{4}.
\label{eq:coeff_contract_part2}
\end{equation}
For the smoothness coefficient in the second line of \eqref{eq:part2_plug_in},
\begin{align}
\frac{1+\beta^2}{2}\cdot\frac{5-\beta^2}{1-\beta^2}
&=
\frac{5+4\beta^2-\beta^4}{2(1-\beta^2)}.
\label{eq:coeff_smooth_exact_part2}
\end{align}
Also,
\begin{align}
\frac{5-\beta^2}{1-\beta^2} - \frac{5+4\beta^2-\beta^4}{2(1-\beta^2)}
&=
\frac{2(5-\beta^2)-(5+4\beta^2-\beta^4)}{2(1-\beta^2)} \nonumber\\
&=
\frac{10-2\beta^2-5-4\beta^2+\beta^4}{2(1-\beta^2)}
=
\frac{\beta^4-6\beta^2+5}{2(1-\beta^2)} \nonumber\\
&=
\frac{(1-\beta^2)(5-\beta^2)}{2(1-\beta^2)}
=
\frac{5-\beta^2}{2}
\;\ge\;0,
\label{eq:coeff_smooth_compare_part2}
\end{align}
hence
\begin{equation}
\frac{5+4\beta^2-\beta^4}{2(1-\beta^2)}
\le
\frac{5-\beta^2}{1-\beta^2}.
\label{eq:coeff_smooth_bound_part2}
\end{equation}

\paragraph{Final bound for Part 2 (no refresh).}
Substituting \eqref{eq:coeff_contract_part2} and \eqref{eq:coeff_smooth_bound_part2} into \eqref{eq:part2_plug_in}, we obtain
\begin{align}
\mathbb{E}\|U_t m_t V_t^\top - \nabla f(w_t)\|^2
&\le \left(1-\frac{1-\beta^2}{4}\right)
\mathbb{E}\left\|U_t m_{t-1}V_t^\top - \nabla f(w_{t-1})\right\|^2 \nonumber\\
&\quad + \frac{5-\beta^2}{1-\beta^2}\,L^2\,\mathbb{E}\|w_t-w_{t-1}\|^2
+ 2(1-\beta)^2\sigma_t^2 + \frac{2(1-\beta)}{1+\beta}\Delta_t.
\label{eq:part2_final}
\end{align}

\subsection{
\texorpdfstring
  {Part 3: Refresh steps and a unified recursion with an $R_t$ term}
  {Part 3: Refresh steps and a unified recursion with an Rt term}
}
\label{subsec:proof_part3}

In this part we incorporate the effect of refreshing $U_t,V_t$ (i.e., $U_t\neq U_{t-1}$ and/or $V_t\neq V_{t-1}$),
and derive a unified one-step recursion that holds for both refresh and non-refresh steps.

\paragraph{Notation.}
Define the tracking error
\begin{equation}
E_t \;:=\; \mathbb{E}\left\|U_t m_t V_t^\top - \nabla f(w_t)\right\|^2.
\label{eq:Et_def}
\end{equation}
Let $\mathbb{I}_{\mathrm{refresh}}$ be the indicator of a refresh step, i.e.,
$\mathbb{I}_{\mathrm{refresh}}=1$ if $t \bmod K = 0$ and $\mathbb{I}_{\mathrm{refresh}}=0$ otherwise.

\paragraph{Step 1: A bound for the \texorpdfstring{$w_{t-1}$}{w\_{t-1}} tracking term under refresh.}
At a refresh step, we need to relate
$\left\|U_t m_{t-1}V_t^\top - \nabla f(w_{t-1})\right\|^2$
to the previous-basis quantity
$\left\|U_{t-1} m_{t-1}V_{t-1}^\top - \nabla f(w_{t-1})\right\|^2$.

Write
\begin{align}
U_t m_{t-1}V_t^\top - \nabla f(w_{t-1})
&=
\underbrace{\Big(U_t m_{t-1}V_t^\top - U_{t-1} m_{t-1}V_{t-1}^\top\Big)}_{=:A_t^{\mathrm{ref}}}
+
\underbrace{\Big(U_{t-1} m_{t-1}V_{t-1}^\top - \nabla f(w_{t-1})\Big)}_{=:B_{t-1}}.
\label{eq:refresh_decomp}
\end{align}
Using the splitting inequality (Young), for any $\gamma>0$,
\begin{equation}
\|A_t^{\mathrm{ref}}+B_{t-1}\|^2
\le
\left(1+\frac{1}{\gamma}\right)\|A_t^{\mathrm{ref}}\|^2
+
(1+\gamma)\|B_{t-1}\|^2.
\label{eq:young_split_refresh}
\end{equation}
Choose
\begin{equation}
\gamma := \frac{1-\beta^2}{8},
\label{eq:gamma_refresh_choice}
\end{equation}
so that
\begin{equation}
1+\gamma = 1+\frac{1-\beta^2}{8},
\qquad
1+\frac{1}{\gamma} = 1+\frac{8}{1-\beta^2}.
\label{eq:gamma_refresh_factors}
\end{equation}
Then \eqref{eq:refresh_decomp}--\eqref{eq:gamma_refresh_factors} imply
\begin{align}
\left\|U_t m_{t-1}V_t^\top - \nabla f(w_{t-1})\right\|^2
\le\;&
\left(1+\frac{8}{1-\beta^2}\right)
\left\|U_t m_{t-1}V_t^\top - U_{t-1} m_{t-1}V_{t-1}^\top\right\|^2 \nonumber\\
&+
\left(1+\frac{1-\beta^2}{8}\right)
\left\|U_{t-1} m_{t-1}V_{t-1}^\top - \nabla f(w_{t-1})\right\|^2.
\label{eq:refresh_track_bound}
\end{align}

\paragraph{Step 2: Plug the refresh bound into the no-refresh recursion from Part 2.}
From Part 2, we already have (for any step $t$) the inequality
\begin{align}
E_t
\le\;&
\left(1-\frac{1-\beta^2}{4}\right)\,
\mathbb{E}\left\|U_t m_{t-1}V_t^\top - \nabla f(w_{t-1})\right\|^2
+
\frac{5-\beta^2}{1-\beta^2}\,L^2\,\mathbb{E}\|w_t-w_{t-1}\|^2 \nonumber\\
&+
2(1-\beta)^2\sigma_t^2
+
\frac{2(1-\beta)}{1+\beta}\Delta_t.
\label{eq:part2_reused}
\end{align}

If $t$ is a refresh step, apply \eqref{eq:refresh_track_bound} inside \eqref{eq:part2_reused}:
\begin{align}
E_t
\le\;&
\left(1-\frac{1-\beta^2}{4}\right)\left(1+\frac{1-\beta^2}{8}\right)
\mathbb{E}\left\|U_{t-1} m_{t-1}V_{t-1}^\top - \nabla f(w_{t-1})\right\|^2 \nonumber\\
&+
\left(1-\frac{1-\beta^2}{4}\right)\left(1+\frac{8}{1-\beta^2}\right)
\mathbb{E}\left\|U_t m_{t-1}V_t^\top - U_{t-1} m_{t-1}V_{t-1}^\top\right\|^2 \nonumber\\
&+
\frac{5-\beta^2}{1-\beta^2}\,L^2\,\mathbb{E}\|w_t-w_{t-1}\|^2
+
2(1-\beta)^2\sigma_t^2
+
\frac{2(1-\beta)}{1+\beta}\Delta_t.
\label{eq:refresh_plugged}
\end{align}

Now we simplify the coefficients \emph{without skipping steps}.
Let $x:=1-\beta^2\in(0,1]$. Then
\begin{align}
\left(1-\frac{x}{4}\right)\left(1+\frac{x}{8}\right)
&=
1+\frac{x}{8}-\frac{x}{4}-\frac{x^2}{32}
=
1-\frac{x}{8}-\frac{x^2}{32}
\le
1-\frac{x}{8}
=
1-\frac{1-\beta^2}{8}.
\label{eq:contract_to_1minus_x_over8}
\end{align}
Also, since $1-\frac{1-\beta^2}{4}\le 1$,
\begin{equation}
\left(1-\frac{1-\beta^2}{4}\right)\left(1+\frac{8}{1-\beta^2}\right)
\le
\left(1+\frac{8}{1-\beta^2}\right).
\label{eq:drop_a_for_R}
\end{equation}
Applying \eqref{eq:contract_to_1minus_x_over8} and \eqref{eq:drop_a_for_R} to \eqref{eq:refresh_plugged}, we get
\begin{align}
E_t
\le\;&
\left(1-\frac{1-\beta^2}{8}\right)
\mathbb{E}\left\|U_{t-1} m_{t-1}V_{t-1}^\top - \nabla f(w_{t-1})\right\|^2
+
\alpha R_t \nonumber\\
&+
\frac{5-\beta^2}{1-\beta^2}\,L^2\,\mathbb{E}\|w_t-w_{t-1}\|^2
+
2(1-\beta)^2\sigma_t^2
+
\frac{2(1-\beta)}{1+\beta}\Delta_t,
\label{eq:refresh_final_with_R}
\end{align}
where we define the refresh mismatch term (kept for later analysis) as
\begin{equation}
R_t
:=
\mathbb{E}\left\|U_t m_{t-1}V_t^\top - U_{t-1} m_{t-1}V_{t-1}^\top\right\|^2.
\label{eq:Rt_def_new}
\end{equation}
\begin{equation}
\alpha ~:=~ 1+\frac{8}{1-\beta^2}.
\label{eq:alpha_def}
\end{equation}

\paragraph{Step 3: A unified recursion for all steps.}
If $t$ is not a refresh step, then $U_t=U_{t-1}$ and $V_t=V_{t-1}$, hence
\begin{equation}
\left\|U_t m_{t-1}V_t^\top - U_{t-1} m_{t-1}V_{t-1}^\top\right\|^2 = 0,
\label{eq:Rt_zero_no_refresh}
\end{equation}
and from \eqref{eq:part2_reused} we also have the stronger contraction
$\left(1-\frac{1-\beta^2}{4}\right)\le \left(1-\frac{1-\beta^2}{8}\right)$.
Therefore, combining both cases yields the unified bound
\begin{align}
E_t
\le\;&
\left(1-\frac{1-\beta^2}{8}\right)
\mathbb{E}\left\|U_{t-1} m_{t-1}V_{t-1}^\top - \nabla f(w_{t-1})\right\|^2
+
\alpha\,\mathbb{I}_{\mathrm{refresh}}\cdot R_t
 \nonumber\\
&+
\frac{5-\beta^2}{1-\beta^2}\,L^2\,\mathbb{E}\|w_t-w_{t-1}\|^2
+
2(1-\beta)^2\sigma_t^2
+
\frac{2(1-\beta)}{1+\beta}\Delta_t.
\label{eq:unified_recursion_part3}
\end{align}
We will keep the term $\mathbb{I}_{\mathrm{refresh}}\cdot R_t$ as is for now and discuss how to further control it later.

\subsection{Part 4: Summing the recursion via a geometric-series lemma}
\label{subsec:proof_part4}

We start from the unified one-step recursion in \eqref{eq:unified_recursion_part3}:
\begin{align}
E_t
\le\;&
\left(1-\rho\right) E_{t-1}
+
\alpha\,\mathbb{I}_{\mathrm{refresh}}\cdot R_t
+
A\,L^2\,\mathbb{E}\|w_t-w_{t-1}\|^2
+
B\,\sigma_t^2
+
C\,\Delta_t,
\label{eq:recursion_part4}
\end{align}
where we denote
\begin{equation}
\rho := \frac{1-\beta^2}{8},\qquad
A := \frac{5-\beta^2}{1-\beta^2},\qquad
B := 2(1-\beta)^2,\qquad
C := \frac{2(1-\beta)}{1+\beta},
\label{eq:coeffs_part4}
\end{equation}
and $E_t:=\mathbb{E}\|U_t m_t V_t^\top - \nabla f(w_t)\|^2$.

\begin{lemma}[Summation of a linear contraction recursion]
\label{lem:recursion_sum_part4}
Let $\{a_t\}_{t\ge 0}$ be a nonnegative sequence satisfying, for some $\rho\in(0,1)$,
\begin{equation}
a_t \le (1-\rho)a_{t-1} + b_t,\qquad \forall t\ge 1,
\label{eq:lemma_recursion_part4}
\end{equation}
where $b_t\ge 0$ for all $t$. Then for any integer $T\ge 1$,
\begin{equation}
\sum_{t=0}^{T-1} a_t
\le
\frac{1}{\rho}\,a_0
+
\frac{1}{\rho}\sum_{t=1}^{T-1} b_t.
\label{eq:lemma_sum_result_part4}
\end{equation}
\end{lemma}

\begin{proof}
Unrolling \eqref{eq:lemma_recursion_part4} yields, for $t\ge 1$,
\begin{equation}
a_t \le (1-\rho)^t a_0 + \sum_{k=1}^{t}(1-\rho)^{t-k} b_k.
\label{eq:lemma_unroll_part4}
\end{equation}
Summing \eqref{eq:lemma_unroll_part4} over $t=1,\dots,T-1$ and adding $a_0$ gives
\begin{align}
\sum_{t=0}^{T-1} a_t
&\le a_0 + a_0\sum_{t=1}^{T-1}(1-\rho)^t
+\sum_{t=1}^{T-1}\sum_{k=1}^{t}(1-\rho)^{t-k} b_k \nonumber\\
&\le a_0\sum_{t=0}^{\infty}(1-\rho)^t
+\sum_{k=1}^{T-1} b_k \sum_{j=0}^{\infty}(1-\rho)^j \nonumber\\
&= \frac{1}{\rho}a_0 + \frac{1}{\rho}\sum_{k=1}^{T-1} b_k,
\end{align}
which proves \eqref{eq:lemma_sum_result_part4}.
\end{proof}

\paragraph{Apply Lemma~\ref{lem:recursion_sum_part4} to \eqref{eq:recursion_part4}.}
We match Lemma~\ref{lem:recursion_sum_part4} by setting
\begin{equation}
a_t := E_t,
\qquad
b_t := \alpha\mathbb{I}_{\mathrm{refresh}}\cdot R_t
+
A\,L^2\,\mathbb{E}\|w_t-w_{t-1}\|^2
+
B\,\sigma_t^2
+
C\,\Delta_t,
\qquad (t\ge 1),
\label{eq:bt_def_part4}
\end{equation}
and $\rho=\frac{1-\beta^2}{8}$ as in \eqref{eq:coeffs_part4}.
Since all terms on the right-hand side are nonnegative, the conditions of Lemma~\ref{lem:recursion_sum_part4} hold.

Therefore, for any $T\ge 1$,
\begin{align}
\sum_{t=0}^{T-1} E_t
\le\;&
\frac{1}{\rho} E_0
+
\frac{1}{\rho}
\sum_{t=1}^{T-1}
\left(
\alpha\mathbb{I}_{\mathrm{refresh}}\cdot R_t
+
A\,L^2\,\mathbb{E}\|w_t-w_{t-1}\|^2
+
B\,\sigma_t^2
+
C\,\Delta_t
\right).
\label{eq:sum_Et_before_simplify}
\end{align}

\paragraph{Simplify the coefficients (explicitly).}
Using $\rho=\frac{1-\beta^2}{8}$, we have $\frac{1}{\rho}=\frac{8}{1-\beta^2}$.
Then each term in \eqref{eq:sum_Et_before_simplify} becomes:
\begin{align}
\frac{1}{\rho}\cdot A
&=
\frac{8}{1-\beta^2}\cdot \frac{5-\beta^2}{1-\beta^2}
=
\frac{8(5-\beta^2)}{(1-\beta^2)^2},
\label{eq:coef_AL2_part4}\\
\frac{1}{\rho}\cdot B
&=
\frac{8}{1-\beta^2}\cdot 2(1-\beta)^2
=
\frac{16(1-\beta)^2}{(1-\beta)(1+\beta)}
=
\frac{16(1-\beta)}{1+\beta},
\label{eq:coef_sigma_part4}\\
\frac{1}{\rho}\cdot C
&=
\frac{8}{1-\beta^2}\cdot \frac{2(1-\beta)}{1+\beta}
=
\frac{16(1-\beta)}{(1-\beta)(1+\beta)^2}
=
\frac{16}{(1+\beta)^2},
\label{eq:coef_Delta_part4}\\
\frac{1}{\rho}\sum_{t=1}^{T-1}\alpha\mathbb{I}_{\mathrm{refresh}}R_t
&=
\frac{8}{1-\beta^2}\sum_{t=1}^{T-1}\alpha\mathbb{I}_{\mathrm{refresh}}R_t.
\label{eq:coef_R_part4}
\end{align}

Substituting \eqref{eq:coef_AL2_part4}--\eqref{eq:coef_R_part4} into \eqref{eq:sum_Et_before_simplify}, we obtain
\begin{align}
\sum_{t=0}^{T-1} \mathbb{E}\left\|U_t m_t V_t^\top - \nabla f(w_t)\right\|^2
\le\;&
\frac{8}{1-\beta^2}\,
\mathbb{E}\left\|U_0 m_0 V_0^\top - \nabla f(w_0)\right\|^2
\label{eq:part4_sum_final}\\
&+
\frac{8}{1-\beta^2}\sum_{t=1}^{T-1}\alpha\mathbb{I}_{\mathrm{refresh}}R_t \nonumber\\
&+
\frac{8(5-\beta^2)L^2}{(1-\beta^2)^2}\sum_{t=1}^{T-1}\mathbb{E}\|w_t-w_{t-1}\|^2 \nonumber\\
&+
\frac{16(1-\beta)}{1+\beta}\sum_{t=1}^{T-1}\sigma_t^2
+
\frac{16}{(1+\beta)^2}\sum_{t=1}^{T-1}\Delta_t. \nonumber
\end{align}
We keep the refresh mismatch contribution $\sum \mathbb{I}_{\mathrm{refresh}}R_t$ for later discussion.

\subsection{
\texorpdfstring
  {Part 5: Descent lemma and the final stationarity bound (keeping $R_t$)}
  {Part 5: Descent lemma and the final stationarity bound (keeping Rt)}
}
\label{subsec:proof_part5}

\paragraph{Assumption (smoothness).}
Assume $f$ is $L$-smooth, i.e., for all $x,y$,
\begin{equation}
f(y)\le f(x)+\langle \nabla f(x),y-x\rangle +\frac{L}{2}\|y-x\|^2.
\label{eq:smoothness_assump_part5}
\end{equation}

\paragraph{Update rule.}
Consider the update
\begin{equation}
w_{t+1}=w_t-\eta\,\widetilde{m}_t,
\qquad\text{where}\qquad
\widetilde{m}_t:=U_t m_t V_t^\top.
\label{eq:update_rule_part5}
\end{equation}

\begin{lemma}[Descent lemma for $w_{t+1}=w_t-\eta\widetilde{m}_t$]
\label{lem:descent_part5}
Under \eqref{eq:smoothness_assump_part5} and \eqref{eq:update_rule_part5}, for any $t\ge 0$,
\begin{align}
f(w_{t+1})
\le\;&
f(w_t)
-\frac{\eta}{2}\|\nabla f(w_t)\|^2
-\left(\frac{1}{2\eta}-\frac{L}{2}\right)\|w_{t+1}-w_t\|^2
+\frac{\eta}{2}\|\widetilde{m}_t-\nabla f(w_t)\|^2.
\label{eq:descent_lemma_part5}
\end{align}
\end{lemma}

\begin{proof}
Apply $L$-smoothness \eqref{eq:smoothness_assump_part5} with $x=w_t$ and $y=w_{t+1}$:
\begin{align}
f(w_{t+1})
&\le f(w_t)+\langle \nabla f(w_t),w_{t+1}-w_t\rangle +\frac{L}{2}\|w_{t+1}-w_t\|^2.
\label{eq:descent_pf_1}
\end{align}
Using \eqref{eq:update_rule_part5}, we have $w_{t+1}-w_t=-\eta\widetilde{m}_t$, hence
\begin{align}
\langle \nabla f(w_t),w_{t+1}-w_t\rangle
&=
-\eta\langle \nabla f(w_t),\widetilde{m}_t\rangle
=
-\eta\left\langle \nabla f(w_t), \nabla f(w_t)+(\widetilde{m}_t-\nabla f(w_t))\right\rangle \nonumber\\
&=
-\eta\|\nabla f(w_t)\|^2-\eta\langle \nabla f(w_t),\widetilde{m}_t-\nabla f(w_t)\rangle.
\label{eq:descent_pf_2}
\end{align}
Apply Young's inequality $-\langle a,b\rangle \le \frac{1}{2}\|a\|^2+\frac{1}{2}\|b\|^2$ to the last term in
\eqref{eq:descent_pf_2} with $a=\nabla f(w_t)$ and $b=\widetilde{m}_t-\nabla f(w_t)$:
\begin{equation}
-\eta\langle \nabla f(w_t),\widetilde{m}_t-\nabla f(w_t)\rangle
\le
\frac{\eta}{2}\|\nabla f(w_t)\|^2+\frac{\eta}{2}\|\widetilde{m}_t-\nabla f(w_t)\|^2.
\label{eq:descent_pf_3}
\end{equation}
Combining \eqref{eq:descent_pf_1}--\eqref{eq:descent_pf_3} gives
\begin{align}
f(w_{t+1})
&\le
f(w_t)
-\eta\|\nabla f(w_t)\|^2
+\frac{\eta}{2}\|\nabla f(w_t)\|^2
+\frac{\eta}{2}\|\widetilde{m}_t-\nabla f(w_t)\|^2
+\frac{L}{2}\|w_{t+1}-w_t\|^2 \nonumber\\
&=
f(w_t)
-\frac{\eta}{2}\|\nabla f(w_t)\|^2
+\frac{\eta}{2}\|\widetilde{m}_t-\nabla f(w_t)\|^2
+\frac{L}{2}\|w_{t+1}-w_t\|^2.
\label{eq:descent_pf_4}
\end{align}
Finally, rewrite $\frac{L}{2}\|w_{t+1}-w_t\|^2
= -\left(\frac{1}{2\eta}-\frac{L}{2}\right)\|w_{t+1}-w_t\|^2 + \frac{1}{2\eta}\|w_{t+1}-w_t\|^2$,
and note that $\frac{1}{2\eta}\|w_{t+1}-w_t\|^2=\frac{\eta}{2}\|\widetilde{m}_t\|^2$.
This yields \eqref{eq:descent_lemma_part5}. (Equivalently, one can directly expand
$\|\widetilde{m}_t\|^2=\|\nabla f(w_t)+(\widetilde{m}_t-\nabla f(w_t))\|^2$ and rearrange.)
\end{proof}

\paragraph{Summing Lemma~\ref{lem:descent_part5}.}
Taking expectation in \eqref{eq:descent_lemma_part5} and summing from $t=0$ to $T-1$ yields
\begin{align}
\frac{\eta}{2}\sum_{t=0}^{T-1}\mathbb{E}\|\nabla f(w_t)\|^2
\le\;&
\mathbb{E}\big[f(w_0)-f(w_T)\big]
+\frac{\eta}{2}\sum_{t=0}^{T-1}\mathbb{E}\|\widetilde{m}_t-\nabla f(w_t)\|^2\\
-\left(\frac{1}{2\eta}-\frac{L}{2}\right)\sum_{t=0}^{T-1}\mathbb{E}\|w_{t+1}-w_t\|^2.
\label{eq:sum_descent_part5}
\end{align}
Define $\Delta f:=\mathbb{E}[f(w_0)-f^\star]$ ($f^\star$ is a lower bound).
Multiply both sides of \eqref{eq:sum_descent_part5} by $\frac{2}{\eta}$ to obtain
\begin{align}
\sum_{t=0}^{T-1}\mathbb{E}\|\nabla f(w_t)\|^2
\le\;&
\frac{2\Delta f}{\eta}
+\sum_{t=0}^{T-1}\mathbb{E}\|\widetilde{m}_t-\nabla f(w_t)\|^2
-\left(\frac{1}{\eta^2}-\frac{L}{\eta}\right)\sum_{t=0}^{T-1}\mathbb{E}\|w_{t+1}-w_t\|^2.
\label{eq:grad_sum_basic_part5}
\end{align}

\paragraph{Substitute the bound from Part 4 (keeping $\sum \mathbb{I}_{\mathrm{refresh}}R_t$).}
From Part 4 (Eq.~\eqref{eq:part4_sum_final}), we have
\begin{align}
\sum_{t=0}^{T-1}\mathbb{E}\|\widetilde{m}_t-\nabla f(w_t)\|^2
\le\;&
\frac{8}{1-\beta^2}\,\mathbb{E}\|\widetilde{m}_0-\nabla f(w_0)\|^2
+
\frac{8}{1-\beta^2}\sum_{t=1}^{T-1}\alpha\mathbb{I}_{\mathrm{refresh}}R_t \nonumber\\
&+
\frac{8(5-\beta^2)L^2}{(1-\beta^2)^2}\sum_{t=1}^{T-1}\mathbb{E}\|w_t-w_{t-1}\|^2
+
\frac{16(1-\beta)}{1+\beta}\sum_{t=1}^{T-1}\sigma_t^2
+
\frac{16}{(1+\beta)^2}\sum_{t=1}^{T-1}\Delta_t.
\label{eq:plug_part4_part5}
\end{align}
Plugging \eqref{eq:plug_part4_part5} into \eqref{eq:grad_sum_basic_part5} yields
\begin{align}
\sum_{t=0}^{T-1}\mathbb{E}\|\nabla f(w_t)\|^2
\le\;&
\frac{2\Delta f}{\eta}
+\frac{8}{1-\beta^2}\,\mathbb{E}\|\widetilde{m}_0-\nabla f(w_0)\|^2
+\frac{8}{1-\beta^2}\sum_{t=1}^{T-1}\alpha\mathbb{I}_{\mathrm{refresh}}R_t \nonumber\\
&+\frac{16(1-\beta)}{1+\beta}\sum_{t=1}^{T-1}\sigma_t^2
+\frac{16}{(1+\beta)^2}\sum_{t=1}^{T-1}\Delta_t \nonumber\\
&+\left(
\frac{8(5-\beta^2)L^2}{(1-\beta^2)^2}
-\left(\frac{1}{\eta^2}-\frac{L}{\eta}\right)
\right)\sum_{t=0}^{T-1}\mathbb{E}\|w_{t+1}-w_t\|^2,
\label{eq:grad_sum_with_steps_part5}
\end{align}
where we aligned indices by the trivial relabeling $t\mapsto t+1$.

\paragraph{A sufficient condition to drop the step-difference term and a concrete parameter choice.}
Recall that the coefficient of $\sum_{t=0}^{T-1}\mathbb{E}\|w_{t+1}-w_t\|^2$ in \eqref{eq:grad_sum_with_steps_part5} is
\begin{equation}
\Gamma
~:=~
\frac{8(5-\beta^2)L^2}{(1-\beta^2)^2}
-\left(\frac{1}{\eta^2}-\frac{L}{\eta}\right).
\label{eq:Gamma_def_part5_replace}
\end{equation}
We aim to enforce $\Gamma\le 0$ so that this entire term is non-positive and can be dropped.

\textbf{Step 1 (remove the dependence on $5-\beta^2$).}
Since $5-\beta^2\le 5$, we have
\begin{equation}
\frac{8(5-\beta^2)L^2}{(1-\beta^2)^2}\le \frac{40L^2}{(1-\beta^2)^2}.
\label{eq:bound_5_beta_replace}
\end{equation}

\textbf{Step 2 (couple $\beta$ with $\eta$ to dominate the $L^2$ term).}
Assume
\begin{equation}
(1-\beta^2)^2 \ge 40L\eta.
\label{eq:beta_eta_coupling_replace}
\end{equation}
Then
\begin{equation}
\frac{40L^2}{(1-\beta^2)^2}\le \frac{L}{\eta},
\label{eq:40_term_to_L_over_eta_replace}
\end{equation}
and thus by \eqref{eq:bound_5_beta_replace},
\begin{equation}
\frac{8(5-\beta^2)L^2}{(1-\beta^2)^2}\le \frac{L}{\eta}.
\label{eq:main_coeff_to_L_over_eta_replace}
\end{equation}

\textbf{Step 3 (ensure the remaining part is non-positive).}
Assume also $\eta\le \frac{1}{2L}$. Then
\begin{equation}
\frac{1}{\eta^2}-\frac{L}{\eta} \ge \frac{1}{\eta^2}-\frac{2L}{\eta}\ge 0.
\label{eq:eta_cond_replace}
\end{equation}
Combining \eqref{eq:main_coeff_to_L_over_eta_replace} and \eqref{eq:eta_cond_replace} yields
\begin{align}
\Gamma
&\le
\frac{L}{\eta}-\left(\frac{1}{\eta^2}-\frac{L}{\eta}\right)
=
-\left(\frac{1}{\eta^2}-\frac{2L}{\eta}\right)
\le 0,
\label{eq:Gamma_nonpos_replace}
\end{align}
so the step-difference term in \eqref{eq:grad_sum_with_steps_part5} can be dropped under
\begin{equation}
\eta\le \frac{1}{2L},
\qquad\text{and}\qquad
(1-\beta^2)^2 \ge 40L\eta.
\label{eq:drop_stepsize_beta_replace}
\end{equation}

\paragraph{Final stationarity bound (keeping $R_t$).}
Under \eqref{eq:drop_stepsize_beta_replace}, dividing \eqref{eq:grad_sum_with_steps_part5} by $T$ gives
\begin{align}
\frac{1}{T}\sum_{t=0}^{T-1}\mathbb{E}\|\nabla f(w_t)\|^2
\le\;&
\frac{2\Delta f}{T\eta}
+\frac{8}{(1-\beta^2)T}\,\mathbb{E}\|\widetilde{m}_0-\nabla f(w_0)\|^2
+\frac{8}{(1-\beta^2)T}\sum_{t=1}^{T-1}\alpha\mathbb{I}_{\mathrm{refresh}}R_t \nonumber\\
&+\frac{16(1-\beta)}{(1+\beta)T}\sum_{t=1}^{T-1}\sigma_t^2
+\frac{16}{(1+\beta)^2T}\sum_{t=1}^{T-1}\Delta_t.
\label{eq:final_avg_grad_bound_part5_replace}
\end{align}
We will next discuss how to control $\sum_{t}\mathbb{I}_{\mathrm{refresh}}R_t$.

\paragraph{Eliminating $\eta$ by a concrete choice of $(\eta,\beta)$.}
To express the bound in terms of $T$, we choose $\beta$ as a function of $\eta$ by saturating
\eqref{eq:beta_eta_coupling_replace}:
\begin{equation}
1-\beta^2 = \sqrt{40L\eta},
\qquad\text{i.e.,}\qquad
\beta^2 = 1-\sqrt{40L\eta}.
\label{eq:beta_choice_replace}
\end{equation}
Moreover, for $\beta\in[0,1]$,
\begin{equation}
1-\beta ~=~ \frac{1-\beta^2}{1+\beta}\le \frac{1-\beta^2}{2}
~=~ \frac{\sqrt{40L\eta}}{2}
~=~ \sqrt{10L\eta}.
\label{eq:one_minus_beta_bound_replace}
\end{equation}

Next, we pick $\eta$ to balance the $T$-dependent terms $\frac{1}{T\eta}$ and $\sqrt{\eta}$, i.e.,
$\frac{1}{T\eta}\asymp \sqrt{\eta}$, which gives $\eta\asymp T^{-2/3}$.
A concrete choice is
\begin{equation}
\eta = \frac{1}{L\,T^{2/3}}.
\label{eq:eta_choice_replace}
\end{equation}
Then
\begin{equation}
\frac{1}{T\eta}=\frac{L}{T^{1/3}},
\qquad\text{and}\qquad
\sqrt{L\eta}=\frac{1}{T^{1/3}}.
\label{eq:eta_scaling_replace}
\end{equation}


\paragraph{An explicit $T$-rate under bounded $\sigma_t^2$, $\Delta_t$, and controlled refresh drift.}
Assume the per-step quantities are uniformly bounded:
\begin{equation}
\sigma_t^2 \le \sigma^2,\qquad \Delta_t \le \Delta,\qquad \forall t.
\label{eq:bounded_sigma_delta}
\end{equation}

Recall from \eqref{eq:final_avg_grad_bound_part5_replace} that the refresh contribution appears as
\begin{equation}
\frac{8\alpha}{(1-\beta^2)T}\sum_{t=1}^{T-1}\mathbb{I}_{\mathrm{refresh}}\,\mathbb{E}R_t,
\label{eq:R_contrib_form_part5}
\end{equation}
where (as in Part 3--4) $\alpha:=1+\frac{8}{1-\beta^2}$ and
\begin{equation}
R_t
:=
\left\|U_t m_{t-1}V_t^\top - U_{t-1} m_{t-1}V_{t-1}^\top\right\|^2.
\label{eq:Rt_def_new_part5}
\end{equation}

\paragraph{Assumption (refresh alignment).}
On refresh steps ($t\bmod K=0$), assume that the lifted core moment is represented by projecting a common lifted matrix
$\widetilde m_{t-1}:=U_{t-1}m_{t-1}V_{t-1}^\top$ onto the refreshed subspaces, i.e.,
\begin{equation}
U_t m_{t-1}V_t^\top
=
(U_tU_t^\top)\,\widetilde m_{t-1}\,(V_tV_t^\top).
\label{eq:refresh_alignment_assump}
\end{equation}
(Equivalently, the core $m_{t-1}$ is re-expressed so that the lifted quantity equals the doubly-projected matrix.)

\paragraph{Assumption (second-order projection drift at refresh).}
On refresh steps ($t\bmod K=0$), the projector drift is controlled as
\begin{equation}
\left\|U_tU_t^\top-U_{t-1}U_{t-1}^\top\right\|_2
\le \kappa_U(1-\beta^2)^2,
\qquad
\left\|V_tV_t^\top-V_{t-1}V_{t-1}^\top\right\|_2
\le \kappa_V(1-\beta^2)^2,
\label{eq:proj_drift_assump_B}
\end{equation}
for some constants $\kappa_U,\kappa_V\ge 0$.

\paragraph{Assumption (bounded lifted moment).}
Assume the lifted moment is uniformly bounded:
\begin{equation}
\mathbb{E}\big\|\widetilde m_{t}\big\|_F^2 \le M^2,\qquad \forall t.
\label{eq:lifted_momentum_bound_B}
\end{equation}

\paragraph{Bounding the refresh mismatch term.}
Let $P_t^U:=U_tU_t^\top$ and $P_t^V:=V_tV_t^\top$. On a refresh step, by
\eqref{eq:refresh_alignment_assump} and the definition of $\widetilde m_{t-1}$,
\begin{align}
U_t m_{t-1}V_t^\top - U_{t-1} m_{t-1}V_{t-1}^\top
&= P_t^U \widetilde m_{t-1} P_t^V - P_{t-1}^U \widetilde m_{t-1} P_{t-1}^V \nonumber\\
&= (P_t^U-P_{t-1}^U)\widetilde m_{t-1} P_t^V
+ P_{t-1}^U \widetilde m_{t-1}(P_t^V-P_{t-1}^V).
\label{eq:Rt_decomp_B}
\end{align}
Using $(a+b)^2\le 2a^2+2b^2$ and $\|A B C\|_F\le \|A\|_2\|B\|_F\|C\|_2$, and noting $\|P_t^U\|_2=\|P_t^V\|_2=1$,
we obtain
\begin{align}
R_t
&\le 2\left\|(P_t^U-P_{t-1}^U)\widetilde m_{t-1} P_t^V\right\|_F^2
   +2\left\|P_{t-1}^U \widetilde m_{t-1}(P_t^V-P_{t-1}^V)\right\|_F^2 \nonumber\\
&\le 2\left\|P_t^U-P_{t-1}^U\right\|_2^2 \left\|\widetilde m_{t-1}\right\|_F^2
   +2\left\|P_t^V-P_{t-1}^V\right\|_2^2 \left\|\widetilde m_{t-1}\right\|_F^2.
\label{eq:Rt_bound_intermediate_B}
\end{align}
Taking expectation and applying \eqref{eq:proj_drift_assump_B}--\eqref{eq:lifted_momentum_bound_B} yields, on refresh steps,
\begin{equation}
\mathbb{E}R_t
\le
2(\kappa_U^2+\kappa_V^2)(1-\beta^2)^4\,M^2.
\label{eq:Rt_bound_B_final}
\end{equation}

\paragraph{Summing over refresh steps.}
Since refresh happens every $K$ steps, we have $\sum_{t=1}^{T-1}\mathbb{I}_{\mathrm{refresh}}\le \left\lceil\frac{T}{K}\right\rceil \le \frac{T}{K}+1$.
Therefore,
\begin{align}
\frac{8\alpha}{(1-\beta^2)T}\sum_{t=1}^{T-1}\mathbb{I}_{\mathrm{refresh}}\,\mathbb{E}R_t
&\le
\frac{8\alpha}{(1-\beta^2)T}\left(\frac{T}{K}+1\right)\cdot 2(\kappa_U^2+\kappa_V^2)(1-\beta^2)^4M^2 \nonumber\\
&\le
\frac{16\alpha(\kappa_U^2+\kappa_V^2)M^2}{K}\,(1-\beta^2)^3
\;+\;
\frac{16\alpha(\kappa_U^2+\kappa_V^2)M^2}{T}\,(1-\beta^2)^3.
\label{eq:R_sum_control_B}
\end{align}
In particular, the dominant refresh contribution scales as $\frac{\alpha}{K}(1-\beta^2)^3$.

\paragraph{Closed-form bound (with the concrete choice of $(\eta,\beta)$).}
With \eqref{eq:eta_choice_replace} and \eqref{eq:beta_choice_replace}, we have
$1-\beta^2 = \sqrt{40L\eta}=\Theta(T^{-1/3})$, and hence
\begin{equation}
\alpha = 1+\frac{8}{1-\beta^2} = \Theta(T^{1/3}),
\qquad
\alpha(1-\beta^2)^3=\Theta(T^{-2/3}).
\label{eq:alpha_scaling_B}
\end{equation}
Substituting \eqref{eq:bounded_sigma_delta} and \eqref{eq:R_sum_control_B} into \eqref{eq:final_avg_grad_bound_part5_replace} yields
\begin{align}
\frac{1}{T}\sum_{t=0}^{T-1}\mathbb{E}\|\nabla f(w_t)\|^2
\le\;&
\frac{2\Delta f}{T\eta}
+\frac{8}{(1-\beta^2)T}\,\mathbb{E}\|\widetilde{m}_0-\nabla f(w_0)\|^2
+\frac{16(1-\beta)}{1+\beta}\cdot \sigma^2
+\frac{16}{(1+\beta)^2}\cdot \Delta \nonumber\\
&+
\frac{16\alpha(\kappa_U^2+\kappa_V^2)M^2}{K}\,(1-\beta^2)^3
\;+\;
\frac{16\alpha(\kappa_U^2+\kappa_V^2)M^2}{T}\,(1-\beta^2)^3.
\label{eq:final_rate_closed_form_B}
\end{align}
Using \eqref{eq:eta_choice_replace}--\eqref{eq:eta_scaling_replace} and \eqref{eq:alpha_scaling_B}, the leading $T$-dependent
terms in \eqref{eq:final_rate_closed_form_B} scale as
\begin{equation}
\frac{2\Delta f}{T\eta} = O(T^{-1/3}),
\qquad
(1-\beta)=O(T^{-1/3}),
\qquad
\frac{\alpha}{K}(1-\beta^2)^3 = O\!\left(\frac{1}{K\,T^{2/3}}\right).
\label{eq:rates_summary_B}
\end{equation}
Therefore, up to constant factors, the stationarity measure satisfies
\begin{equation}
\frac{1}{T}\sum_{t=0}^{T-1}\mathbb{E}\|\nabla f(w_t)\|^2
=
O\!\left(\frac{1}{T^{1/3}}\right)
\;+\;
O(\Delta)
\;+\;
O\!\left(\frac{1}{K\,T^{2/3}}\right)
\;+\; \text{(transient terms)}.
\label{eq:final_rate_bigO_B}
\end{equation}

\begin{algorithm}[tb]
\caption{TSR-SGD (Momentum w/o weight decay)}
\label{alg:tsr-sgd}
\centering
\begin{minipage}{0.97\columnwidth}
\small
\begin{algorithmic}
\STATE {\bfseries Input:} rank $r_\ell$, refresh period $K_\ell$, oversampling $p$,
SGD hyperparams $(\eta,\beta)$
\STATE {\bfseries State:} bases $(U,V)$, core momentum $m$, step $t\leftarrow 1$
\STATE Set $r\leftarrow r_\ell$, $K\leftarrow K_\ell$, $k\leftarrow r+p$.
Initialize $m\leftarrow 0$, and $(U,V)$ by one refresh.
Let $\AR(\cdot)$ denote all-reduce averaging across workers.
\STATE (Algorithm corresponds to the update rule analyzed in Theorem~\ref{thm:tsr_main}.)

\REPEAT
  \STATE Each worker $i$ computes local gradient $G_{t,i}$

  \STATE \textcolor{gray}{// Randomized Subspace Refresh (identical to TSR-Adam)}
  \IF{$t\bmod K = 0$}
    \STATE Sample shared $\Omega$ (shared RNG seed)
    \STATE $Y_{t,i}\leftarrow G_{t,i}\Omega$
    \STATE $Q_{t,i} \leftarrow \mathrm{orth}(Y_{t,i})$
    \STATE $Y^{\mathrm{row}}_{t,i}\leftarrow G_{t,i}^\top Q_{t,i}$
    \STATE $Q_{t,i}^{\mathrm{row}} \leftarrow \mathrm{orth}(Y^{\mathrm{row}}_{t,i})$
    \STATE $Y_{t,i}\leftarrow G_{t,i}Q_{t,i}^{\mathrm{row}}$
    \STATE $Q_{t,i} \leftarrow \mathrm{orth}(Y_{t,i})$
    \STATE $B_{t,i}\leftarrow Q_{t,i}^\top G_{t,i},\quad B_t \leftarrow \AR(B_{t,i})$
    \STATE $Q_t \leftarrow \AR(Q_{t,i})$
    \STATE $B_t=\widetilde{U}\Sigma\widetilde{V}^\top$ \hfill (SVD on $k\times n$)
    \STATE $U\leftarrow Q\,\widetilde{U}_{[:,1:r]},\quad V\leftarrow \widetilde{V}_{[:,1:r]}$
  \ENDIF

  \STATE \textcolor{gray}{// Core Synchronization}
  \STATE $C_{t,i}\leftarrow U^\top G_{t,i} V,\quad C_t \leftarrow \AR(C_{t,i})$

  \STATE \textcolor{gray}{// Momentum SGD Update in Core Space}
  \STATE $m \leftarrow \beta m + (1-\beta)C_t$
  \STATE $D \leftarrow m$

  \STATE \textcolor{gray}{// Lift and Weight Update (no weight decay)}
  \STATE $\Delta W_t \leftarrow U D V^\top$
  \STATE $W \leftarrow W - \eta\,\Delta W_t$
  \STATE $t \leftarrow t+1$
\UNTIL{convergence criteria met}
\end{algorithmic}
\end{minipage}
\end{algorithm}

\clearpage
\section{More about Experiments}
\label{app:exp}
\subsection{Pre-Training Hyperparameters}
We introduce details of the LLaMA architecture and hyperparameters used for pre-training. Table 5 shows the most
hyperparameters of LLaMA models across model sizes.We use a max sequence length of 256 for all models,with a batch
size of 262K tokens for LLaMA-60M and LLaMA-130M , and a batch size of 131K tokens for LLaMA-350M and LLaMA-1B. For all experiments,we adopt learning rate warmup for the first 10\% of the training steps,and use
cosine annealing for the learning rate schedule, decaying to 10\% of the initial learning rate.For all model sizes, we use a learning rate of 0.01. For the LLaMA-60M model, we apply a scaling factor of 0.5, while the remaining models use a scaling factor of 0.75.
\begin{table*}[ht]
  \caption{Hyperparameters of LLaMA models for evaluation.}
  \label{apptab:hyperparameters}
  \begin{center}
    \begin{small}
      \begin{sc}
        \begin{tabular}{ccccccc}
          \toprule
          Params &Vocabulary &Hidden & Intermediate & Heads & Layers&Steps \\
          \midrule
          60M & 32000&512 & 1376 & 8 & 8& 20K\\
          130M & 32000&768 & 2048 & 12 & 12& 20K\\
          350M & 32000&1024 & 2736 & 16 & 24& 90K\\
          1B & 32000&52048 & 5461 & 32 & 24& 90K\\
          \bottomrule
        \end{tabular}
      \end{sc}
    \end{small}
  \end{center}
  \vskip -0.1in
\end{table*}
\subsection{More experiment results}
We present additional experimental results in Table 6.
\begin{table*}[th]
  \caption{More Experiment Results}
  \label{app_tab:hyper_llama}
  \begin{center}
  \setlength{\tabcolsep}{3pt}
    \begin{small}
      \begin{sc}
        \begin{tabular}{lccccccc}
          \toprule
          Scale & Method & Rank & $K$ & Final loss$\downarrow$ & Bytes/step$\downarrow$ & Peak Bytes &  Memory \\
          \midrule
          60M  & TSR & 128(64) & 200 & 3.82 & 0.008G & 0.05G & 0.14G  \\
          60M(40k steps)  & TSR & 256(64) & 100 & 3.56 & 0.020G & 0.10G & 0.17G  \\
          130M & TSR & 256(96) & 50 & 3.50 & 0.032G & 0.20G & 0.38G\\
          350M & TSR & 256(128) & 50 & 3.30 & 0.062G & 0.52G & 1.02G\\
          \bottomrule
        \end{tabular}
      \end{sc}
    \end{small}
  \end{center}
  \vskip -0.1in
\end{table*}

\subsection{Details of Fine-Tuning on GLUE}
We fine-tune the pre-trained RoBERTa-Base model on the GLUE benchmark.We trained the model for 30 epochs with a batch size of 16 for all tasks except for CoLA, which uses a batch size
of 32.We tune the learning rate and scale factor for TSR. Table 7 shows the hyperparameters used for fine-tuning
RoBERTa-Base for TSR.

\begin{figure*}[t]
  \vskip 0.1in
  \begin{center}
    \centerline{\includegraphics[width=\textwidth]{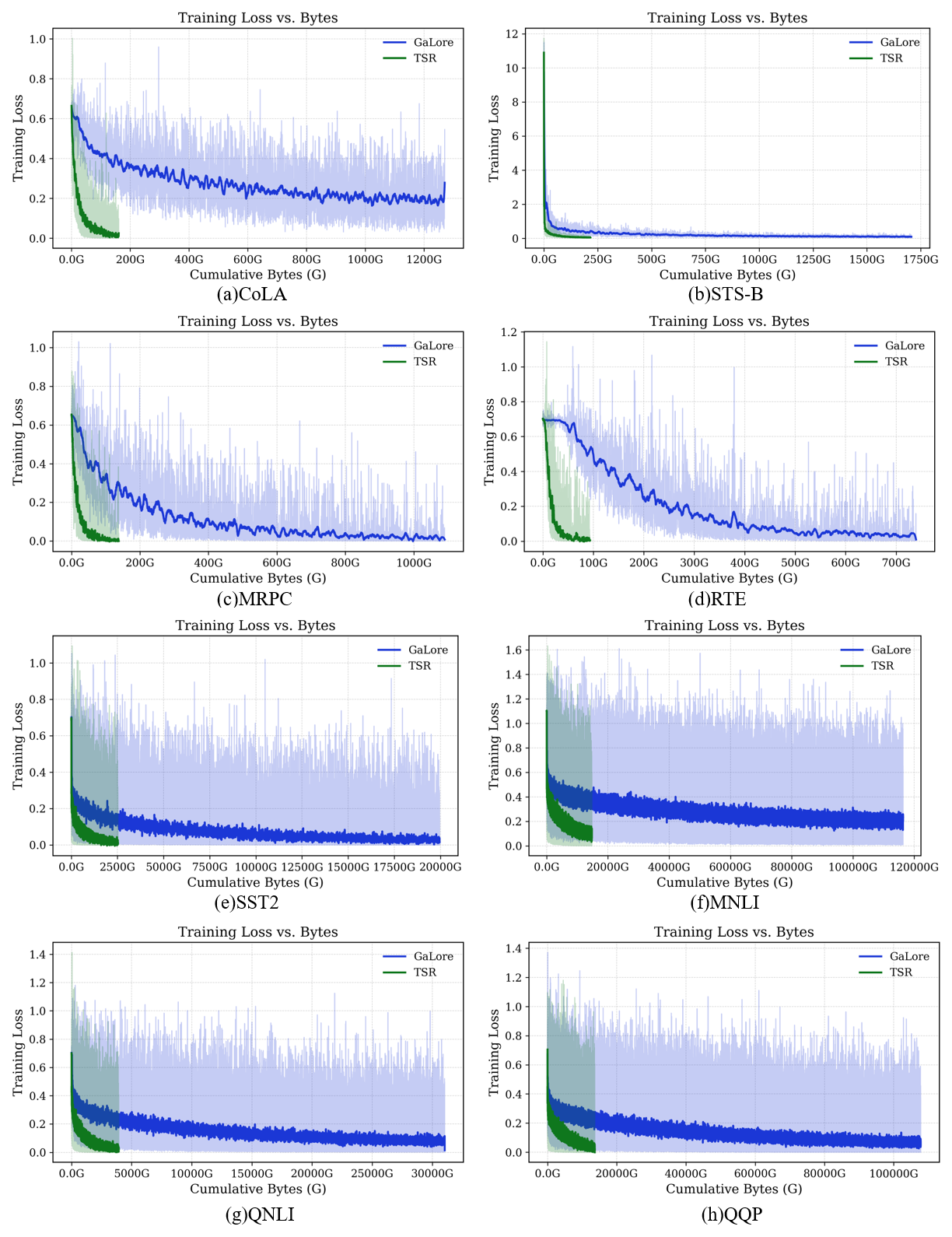}}
    \caption{
      \textbf{Fine-tuning Loss-Byte curves on GLUE tasks .} 
    }
    \label{app_fig:SFT}
  \end{center}
  \vskip -0.1in
\end{figure*}

\begin{table*}[th]
  \caption{Hyperparameters of fine-tuning RoBERTa-Base for TSR.}
  \label{app_tab:finetune}
  \begin{center}
  \setlength{\tabcolsep}{2pt}
    \begin{small}
      \begin{sc}
        \begin{tabular}{lcccccccc}
          \toprule 
           &CoLA&STS-B & MRPC &RTE&SST2 &MNLI& QNLI&QQP \\
          \midrule
          Batch Size&32&16&16&16&16&16&16&16\\
          Epochs&30&30&30&30&30&30&30&30\\
          Learning Rate&$1E-05$&$3E-05$&$2E-05$&$2E-05$&$2E-05$&$1E-05$&$1E-05$&$2E-05$\\
          Scaling Factor&8&2&2&2&2&2&2&2\\
          Max Seq. Len.&&&&512&&&&\\
          \bottomrule
        \end{tabular}
      \end{sc}
    \end{small}
  \end{center}
  \vskip -0.1in
\end{table*}


\end{document}